\documentclass{article}

\usepackage{subfigure}
\usepackage{subfigure,graphicx,dsfont,amssymb,amsmath,calc,tabulary}

\usepackage{amsmath}
\usepackage{amssymb}
\usepackage{amsthm}
\usepackage{fixmath}
\usepackage{algorithm}
\usepackage{algorithmic}
\usepackage[numbers]{natbib}
\usepackage{xcolor}
\newtheorem{theorem}{Theorem}

\newtheorem{lemma}[theorem]{Lemma}

\usepackage{amsmath,amsfonts,bm}

\def\Algref#1{Algorithm~\ref{#1}}

\def\1{\bm{1}}

\def\vtheta{{\bm{\theta}}}

\def\vq{{\bm{q}}}

\def\vw{{\bm{w}}}
\def\vx{{\bm{x}}}

\def\vz{{\bm{z}}}

\DeclareMathAlphabet{\mathsfit}{\encodingdefault}{\sfdefault}{m}{sl}
\SetMathAlphabet{\mathsfit}{bold}{\encodingdefault}{\sfdefault}{bx}{n}

\def\gB{{\mathcal{B}}}
\def\gC{{\mathcal{C}}}

\def\gL{{\mathcal{L}}}

\def\gO{{\mathcal{O}}}

\def\gR{{\mathcal{R}}}
\def\gS{{\mathcal{S}}}

\def\sH{{\mathbb{H}}}

\def\sK{{\mathbb{K}}}

\def\sR{{\mathbb{R}}}

\newcommand{\E}{\mathbb{E}}

\DeclareMathOperator*{\argmin}{arg\,min}

\usepackage[final]{neurips_2022}

\usepackage[utf8]{inputenc} 
\usepackage[T1]{fontenc}    
\usepackage{hyperref}       
\usepackage{url}            
\usepackage{booktabs}       
\usepackage{amsfonts}       
\usepackage{nicefrac}       
\usepackage{microtype}      
\usepackage{xcolor}         

\title{Personalized Online Federated Learning \\with Multiple Kernels}

\author{%
  Pouya M.~Ghari \\
  University of California Irvine\\
  \texttt{pmollaeb@uci.edu}
  \And
  Yanning Shen \thanks{Corresponding author} \\
  University of California Irvine\\
  \texttt{yannings@uci.edu}
}

\begin{document}

\maketitle

\begin{abstract}
Multi-kernel learning (MKL) exhibits well-documented performance in online non-linear function approximation. Federated learning enables a group of learners (called clients) to train an MKL model on the data distributed among clients to perform online non-linear function approximation. There are some challenges in online federated MKL that need to be addressed: i) Communication efficiency especially when a large number of kernels are considered ii) Heterogeneous data distribution among clients. The present paper develops an algorithmic framework to enable clients to communicate with the server to send their updates with affordable communication cost while clients employ a large dictionary of kernels. Utilizing random feature (RF) approximation, the present paper proposes scalable online federated MKL algorithm. We prove that using the proposed online federated MKL algorithm, each client enjoys sub-linear regret with respect to the RF approximation of its best kernel in hindsight, which indicates that the proposed algorithm can effectively deal with heterogeneity of the data distributed among clients. Experimental results on real datasets showcase the advantages of the proposed algorithm compared with other online federated kernel learning ones.
\end{abstract}

\section{Introduction}
Kernel learning exhibits well-documented performance in function approximation tasks, while providing theoretical guarantees associated with different performance metrics, see e.g. \cite{Wahba1990,Hoi2013,Rudi2017}. In some cases, a group of learners aims at collaborating to perform function approximation without revealing their data. To this end, federated learning has been emerged as a crucial learning paradigm by enabling a group of learners called clients to collaborate with each other by communicating with a central server to train a centralized model \cite{McMahan2017,Li2020,Fallah2020,Kairouz2021}. Through this process, clients send model parameters and updates to the server without revealing their data. Upon receiving updates from clients, the server updates the model. Therefore, federated learning enables clients to perform kernel learning for function approximation. In this context, a server and clients collaborate with each other to learn the optimal kernel. Furthermore, in some practical cases, clients may need to perform the function approximation in an online fashion while they are collaborating with the server to learn the kernel. For example, consider the case where clients may not have enough memory to store data in batch. In addition, data samples may arrive in a  sequential manner such that clients are not able to perform the function approximation in batch form. 
There are major challenges in performing online kernel learning in federated fashion that need to be addressed:\\
\textbf{Communication Efficiency:} Communication efficiency arises as a bottleneck in federated learning (see e.g. \cite{konency2017,Rothchild2020,Hamer2020,Ghari2022}). Specifically, limited clients-to-server communication bandwidth restricts the number of parameters that can be sent from clients to the server.\\
\textbf{Heterogeneous Data:} The distribution of data observed by a client might be different from others (see e.g. \cite{Smith2017,Li2021,Collins2021}). Thus, the optimal kernel that is aimed to be learned is different across clients.\\
\textbf{Computational Complexity:} Clients should be able to perform function approximation fast enough in order to make a decision in real-time. Therefore, the computational complexity of kernel learning methods should be affordable for clients.

Conventional online kernel learning approaches (see e.g. \cite{Hoi2013,Sahoo2014}) suffer from `curse of dimensionality' \cite{Bengio2006} in the sense that the number of parameters that should be learned increases with the number of observed data. This can make employing conventional online kernel learning approaches infeasible to perform online federated kernel learning since clients may be required to send a large number of parameter updates to the server while the available clients-to-server communication bandwidth is not enough for sending such information. Approximating kernels by finite-dimensional feature representations (e.g. Nystr\"{o}m method \cite{Williams2000} and random feature method of \cite{Rahimi2007}) makes online kernel learning approaches scalable in the sense that the learner can choose the number of parameters that should be learned, independent of the number of observed data samples (see e.g. \cite{Lu2016,Bouboulis2018,Zhang2019}). Employing finite-dimensional feature representations of kernels to perform online federated kernel learning, clients can choose the number of parameters that they should send to the server. Therefore, finite-dimensional kernel approximation can better cope with limited clients-to-server communication bandwidth compared to conventional kernel learning approaches. Random feature (RF) approximation \cite{Rahimi2007} has been exploited to perform online federated kernel learning when a single pre-selected kernel function is employed \cite{Kuh2021,Gogineni2022}. The choice of the kernel function greatly affects the performance of function approximation when it comes to exploiting only a single kernel function. Employing multiple kernels instead of a single pre-selected one, can lead to obtaining more accurate function approximation since multi-kernel learning (MKL) can learn combination of kernels \cite{Kloft2011}. Online federated MKL algorithms with theoretical guarantees called vM-KOFL and eM-KOFL have been proposed in \cite{Hong2021}. However, eM-KOFL and vM-KOFL do not provide personalized MKL models for clients since they learn the same combination of kernels for all clients.

The present paper proposes a novel \textbf{p}ersonalized \textbf{o}nline \textbf{f}ederated \textbf{MKL} algorithm called POF-MKL that provides a personalized MKL model for each client while it is ensured that the available clients-to-server communication bandwidth can afford communication cost of sending clients' updates to the server. In order to alleviate the communication cost of MKL, the propsoed POF-MKL employs RF approximation of kernels and at each time instant, each client chooses a subset of kernels to send their updates to the server instead of sending the updates of all kernels. The number of kernels in the chosen subset is selected such that the required bandwidth to send all clients' updates does not exceed the available clients-to-server communication bandwidth. Therefore, clients can send their updates to the server independent of the number of kernels in the dictionary and as a result a comparatively large dictionary of kernels can be considered to perform function approximation. Contributions of the present paper can be summarized as follows:\\
\textbf{c1.} Leveraging the proposed POF-MKL, clients can update a subset of kernels' parameters which alleviates computational complexity and communication cost of sending updates to the server;\\ 
\textbf{c2.} 
Through theoretical analysis, it is proved that using the proposed POF-MKL, each client achieves sub-linear regret with respect to RF approximation of the best kernel in hindsight associated with the corresponding client data samples (c.f. Theorem \ref{th:2}). Moreover, it is guaranteed that the server achieves sub-linear regret with respect to the best function approximator (c.f. Theorem \ref{th:3});\\ 
\textbf{c3.} Experiments on real datasets showcase the effectiveness of the proposed POF-MKL compared to existing online federated kernel learning algorithms.

\section{Problem Statement and Preliminaries} \label{problem_statement}
Let there be a set of $K$ clients performing function approximation task in an online fashion. The $k$-th client's goal is to learn the function $f$ using the stream of data samples $\{(\vx_{k,t}, y_{k,t})\}_{t=1}^T$ such that $\vx_{k,t} \in \sR^d$ is the data sample observed by the $k$-th client at time $t$ and $y_{k,t}$ is the label associated with $\vx_{k,t}$. In the kernel learning context, the function $f$ is assumed to belong to a reproducing kernel Hilbert space (RKHS). The present paper studies the personalized federated supervised function approximation problem 
\begin{align}
    \min_{f \in \sH} \sum_{t=1}^{T}{\sum_{k=1}^{K}{\gL(f(\vx_{k,t}),y_{k,t})}} \label{eq:1}
\end{align}
where $\sH$ represents the RKHS the function $f$ belongs to and $\gL(\cdot,\cdot)$ denotes the loss function which can be defined as
\begin{align}
    \gL(f(\vx),y) = \gC(f(\vx),y) + \lambda \Omega(\|f\|^2) \label{eq:2}
\end{align}
where $\gC(\cdot,\cdot)$ is the cost function (e.g. least squares function for regression task), $\lambda$ denotes the regularization coefficient and $\Omega(\cdot)$ represents a regularizer function to prevent over-fitting and control the model complexity.
Let $\boldsymbol{\Theta}$ be the global parameters of the function $f$ which are learned through collaboration of clients with the server while $\vw_k$ be the personalized parameter of the function $f$ learned locally by the $k$-th client. Thus, the goal is that the cumulative difference between $f(\vx_{k,t};\boldsymbol{\Theta},\vw_k)$ and $y_{k,t}$ over time is minimized. At each time instant $t$, upon observing the data sample $\vx_{k,t}$, the $k$-th client make the prediction $f(\vx_{k,t};\boldsymbol{\Theta},\vw_k)$ and then observes the true label $y_{k,t}$. Therefore, the function approximation problem that the server aims at solving can be expressed as $\min_{\Theta} \sum_{t=1}^{T}{\sum_{k=1}^{K}{\gL(f(\vx_{k,t};\boldsymbol{\Theta},\vw_k),y_{k,t})}}$.
Moreover, finding the local parameters $\vw_k$ by the $k$-th client can be expressed as the optimization problem
$
    \min_{\vw_k} \sum_{t=1}^{T}{\gL(f(\vx_{k,t};\boldsymbol{\Theta},\vw_k),y_{k,t})}. 
$
In order to perform the function approximation task in an online fashion, the $k$-th client needs to perform the task with the values of $\boldsymbol{\Theta}$ and $\vw_k$ at time $t$ denoted by $\boldsymbol{\Theta}_t$ and $\vw_{k,t}$, respectively. Thus, the values of function parameters $\boldsymbol{\Theta}$ and $\vw_k$ should be updated `on the fly'. 
Since the function $f(\cdot;\cdot,\cdot)$ belongs to a reproducing kernel Hilbert space (RKHS), based on the representer theorem \citep{Wahba1990}, given data samples, the optimal solution for \eqref{eq:1} can be obtained as
\begin{align}
    \hat{f}(\vx) = \sum_{t=1}^{T}{\sum_{k=1}^{K}{\alpha_{k,t} \kappa(\vx,\vx_{k,t})}} \label{eq:3}
\end{align}
where $\kappa(\cdot,\cdot)$ denotes symmetric positive definite kernel function such that $\kappa(\vx,\vx^\prime)$ measures the similarity between $\vx$ and $\vx^\prime$. And $\alpha_{k,t}$ is an unknown coefficient associated with $\kappa(\vx,\vx_{k,t})$ which is required to be estimated. In this case, $\hat{f}(\cdot)$ in \eqref{eq:3} belongs to the RKHS $\sH:= \{f(\cdot)|f(\vx)=\sum_{t=1}^{\infty}{\sum_{k=1}^{K}{{\alpha}_{k,t}\kappa(\vx,\vx_{k,t})}}\}$ such that RKHS norm is defined as $\| f\|_{\sH}^{2} := \sum_{t}{\sum_{t^\prime}{\alpha_{t}\alpha_{t^\prime}{\kappa(\vx_{t},\vx_{{t}^{\prime}})}}}$. Furthermore, from \eqref{eq:3}, it can be inferred that $\boldsymbol{\Theta}_t = [\alpha_{1,1},\ldots,\alpha_{K,1},\ldots,\alpha_{1,t},\ldots,\alpha_{K,t}]$. Therefore, the number of coefficients $\{\alpha_{k,\tau}\}_{\tau=1}^{t}$, $\forall k$ that should be estimated to obtain $\hat{f}(\cdot)$ increases over time. This is known as \emph{curse of dimensionality} \citep{Wahba1990} since the computational complexity of function approximation increases with time. This brings challenge for federated implementation of function approximation since dimension of updates that should be sent to the server by each client grows over time and when $T$ is large, the available communication bandwidth may not be enough for clients to send their updates. 

In order to deal with the increasing number of unknown coefficients, one can employ random Fourier approximation \citep{Rahimi2007}. Assume that $\kappa(\cdot)$ is a shift-invariant kernel meaning that $\kappa(\vx,\vx^\prime)=\kappa(\vx-\vx^\prime)$. Let $\pi_{\kappa}(\boldsymbol{\rho})$ denotes the Fourier transform of $\kappa(\cdot)$. If the kernel function $\kappa(\cdot)$ is normalized such that $\kappa(\boldsymbol{0})=1$, then $\pi_{\kappa}(\boldsymbol{\rho})$ can be viewed as a probability density function (PDF) (see e.g. \cite{Rahimi2007}). 
Let $\boldsymbol{\rho}_1,\ldots,\boldsymbol{\rho}_D$ be a set of $D$ independent and identically distributed (i.i.d) vectors drawn from $\pi_{\kappa}(\cdot)$. 
Let the vector $\vz(\vx)$ be defined as
\begin{align}
    \vz(\vx) = \frac{1}{\sqrt{D}}[\sin(\boldsymbol{\rho}_1^\top \vx),\ldots,\sin(\boldsymbol{\rho}_D^\top \vx),\cos(\boldsymbol{\rho}_1^\top \vx),\ldots,\cos(\boldsymbol{\rho}_D^\top \vx)]. \label{eq:6}
\end{align}
Then, $\hat{\kappa}_{r}(\vx-\vx^\prime) = \vz(\vx)^\top \vz(\vx^\prime)$ constitutes an unbiased estimator of $\kappa(\vx-\vx^\prime)$ and the random feature (RF) approximation of $\hat{f}(\vx)$ in \eqref{eq:3} can be obtained as
\begin{align}
    \hat{f}_{\text{RF}}(\vx) = \sum_{t=1}^{T}{\sum_{k=1}^{K}{\alpha_{k,t} \vz(\vx_{k,t})^\top \vz(\vx)}} := \boldsymbol{\theta}^\top \vz(\vx) \label{eq:7}
\end{align}
where in this case $\boldsymbol{\theta} = \sum_{t=1}^{T}{\sum_{k=1}^{K}{\alpha_{k,t} \vz(\vx_{k,t})}}$. According to \eqref{eq:6}, $\vz(\vx_{k,t})$ is a $2D$ vector and as a result it can be concluded that $\boldsymbol{\theta}$ is a $2D$ vector as well. Therefore, using RF approximation, the vector $\boldsymbol{\theta}$ should be estimated whose dimension does not grow over time. 

The performance of a kernel learning algorithm depends on the choice of the kernel. Thus, performing the function approximation using a pre-selected kernel requires prior information which may not be available. To cope with this, employing a dictionary of kernels in lieu of a pre-selected single kernel has been proposed in the literature (see e.g. \cite{Sonnenburg2006,Kloft2011,Lv2021}). Specifically, the kernel is learned as a combination of kernels in the dictionary. Let $\kappa_1(\cdot),\ldots,\kappa_N(\cdot)$ be a set of $N$ kernels where $\kappa_i(\cdot)$ denotes the $i$-th kernel. The function $\Bar{\kappa}(\cdot)$ belongs to the convex hull $\sK := \{\Bar{\kappa}(\vx) = \sum_{i=1}^{N}{\beta_i \kappa_i(\vx)}, \beta_i \ge 0, \forall i, \sum_{i=1}^{N}{\beta_i}=1\}$ is a kernel \cite{Scholkopf2001}. Therefore, in online multi-kernel learning, the goal is to learn the convex combination of kernels in the dictionary to minimize the cumulative regret with respect to the best function approximator in hindsight. The cumulative regret is defined as the cumulative difference between loss of the online multi-kernel learning algorithm and that of the best function approximator in hindsight. Furthermore, for a dataset $\{(\vx_t,y_t)\}_{t=1}^T$, the best function approximator is $f^*(\cdot) \in \argmin_{f_i^*, i\in [N]} \sum_{t=1}^{T}{\gL(f_i^*(\vx_t),y_t)}$ where $f_i^*(\cdot) \in \argmin_{f \in \sH_i} \sum_{t=1}^{T}{\gL(f(\vx_t),y_t)}$ such that 
$\sH_i$ is an RKHS induced by $\kappa_i(\cdot)$ and $[N]:=\{1,\ldots,N\}$. Enabled by random feature approximation, centralized and scalable online multi-kernel learning algorithms have been proposed in literature (see e.g. \cite{Sahoo2019,Shen2019,Ghari2020}). The present paper proposes an algorithmic framework for personalized online federated MKL using RF approximation of kernels in the dictionary. 

\section{Personalized Online Federated Multi-Kernel Learning} \label{POF-MKL}
The present section proposes an algorithmic framework for online federated multi-kernel learning which can deal with heterogeneous data among clients. To perform function approximation, RF approximations of kernel functions are employed. For the $i$-th kernel $\kappa_i$, vectors $\boldsymbol{\rho}_{i,1},\ldots,\boldsymbol{\rho}_{i,D}$ are drawn i.i.d from $\pi_{\kappa_i}(\cdot)$ to construct the random feature vector $\vz_i(\vx) = \frac{1}{\sqrt{D}}[\sin(\boldsymbol{\rho}_{i,1}^\top \vx),\ldots,\sin(\boldsymbol{\rho}_{i,D}^\top \vx),\cos(\boldsymbol{\rho}_{i,1}^\top \vx),\ldots,\cos(\boldsymbol{\rho}_{i,D}^\top \vx)]$. Then, at time instant $t$, the random feature approximation associated with $\kappa_i(\cdot)$ can be obtained as $\hat{f}_{\text{RF},it}(\vx) = \boldsymbol{\theta}_{i,t}^\top \vz_i(\vx)$ where $\boldsymbol{\theta}_{i,t}$ is the global function parameter associated the $i$-th kernel at time $t$.

\subsection{Algorithm}
At each time instant $t$, the server transmits global function parameters $\boldsymbol{\theta}_{i,t}$, $\forall i \in [N]$ to all clients. The $k$-th client, assigns the weight $w_{ik,t}$ to the $i$-th kernel which indicates the confidence of the $k$-th client at time $t$ in the function approximation given by the $i$-th kernel. Upon receiving new data sample $\vx_{k,t}$, the $k$-th client performs the function approximation combining kernels' RF approximations as
\begin{align}
    \hat{f}(\vx_{k,t};\hat{\boldsymbol{\Theta}}_t,\vw_{k,t}) = \sum_{i=1}^{N}{\frac{w_{ik,t}}{W_{k,t}}\boldsymbol{\theta}_{i,t}^\top \vz_i(\vx_{k,t})} = \sum_{i=1}^{N}{\frac{w_{ik,t}}{W_{k,t}}\hat{f}_{\text{RF},it}}(\vx_{k,t};\vtheta_{i,t}) \label{eq:9}
\end{align}
where $\hat{\boldsymbol{\Theta}}_t = [\boldsymbol{\theta}_{1,t},\ldots,\boldsymbol{\theta}_{N,t}]$, $\vw_{k,t}=[w_{1k,t},\ldots,w_{Nk,t}]$ and $W_{k,t} = \sum_{i=1}^{N}{w_{ik,t}}$. As it can be inferred from \eqref{eq:9}, each client constructs its own personalized combination of kernels. Upon observing the true label $y_{k,t}$, the $k$-th client calculates the losses $\gL(\hat{f}_{\text{RF},it}(\vx_{k,t};\vtheta_{i,t}),y_{k,t})$, $\forall i \in [N]$. Then, the $k$-th client leverages calculated losses to locally update both global and local parameters. Let $\boldsymbol{\theta}_{ik,t+1}$ and $w_{ik,t+1}$ denote the $k$-th client's local updates of $\boldsymbol{\theta}_{i,t}$ and $w_{ik,t}$, respectively. Specifically, the $k$-th client utilizes multiplicative update rule to update $w_{ik,t}$ as
\begin{align}
    w_{ik,t+1} = w_{ik,t}\exp\left(-\eta_k \gL(\hat{f}_{\text{RF},it}(\vx_{k,t};\vtheta_{i,t}),y_{k,t})\right), \forall i \in [N] \label{eq:12}
\end{align}
where $\eta_k$ is the learning rate of the $k$-th client. Note that the $k$-th client ($\forall k \in [K]$) does not send its updated local parameter $\vw_{k,t+1}$ to the server. Clients send their locally updated global parameters to the server (i.e. $\boldsymbol{\theta}_{ik,t+1}$). Aggregating local updates, the server updates global parameters to $\hat{\boldsymbol{\Theta}}_{t+1}$. If all clients send updates associated with all kernels (i.e. $\boldsymbol{\theta}_{ik,t+1}$, $\forall i \in [N]$), this requires sending $2ND$ parameters by each client at each time instant. When the number of both clients and kernels are large, the available client-to-server communication bandwidth may not be enough to afford sending $2NDK$ parameters per time instant even for small values of $D$. Note that reducing $N$ and $D$ degrade the performance of online federated MKL. Reducing $N$ (the number of kernels), decreases the flexibility of clients to construct their ideal kernel using convex combination of kernels in the dictionary. Reducing $D$ can degrade the accuracy of RF approximation.

The present paper proposes an algorithmic framework to enable clients to perform online function approximation with sufficiently large dictionary of kernels while the available clients-to-server communication bandwidth can afford sending updates from clients to the server when a desirable value for the number of random features $D$ is chosen. To this end, at each time instant, each client randomly chooses a subset of $M$ kernels among all $N$ kernels in the dictionary. Then, each client updates and sends the global parameters of the chosen $M$ kernels to the server instead of updating and sending global parameters of all kernels. To choose a subset of $M$ kernels, each client splits kernels into some bins and draws randomly one of the bins at each time instant. Each bin contains at most $M$ kernels and each client updates and sends global parameters associated with kernels in the chosen bin. In order to distribute kernels among bins, at first the $k$-th client sorts kernels in descending order according to kernels' weights $\{w_{ik,t}\}_{i=1}^{N}$. Let $\gB_j$ represents the $j$-th bin of kernels. The $k$-th client adds kernels from sorted list one by one to $\gB_j$ until either all kernels are assigned to a bin or the number of kernels in $\gB_j$ reaches $M$. When there are some kernels that are not assigned to any bins while there are $M$ kernels in $\gB_j$, the $k$-th client opens the bin $\gB_{j+1}$ and adds kernels to this bin. This continues until all kernels are assigned to a bin. As it can be inferred from the procedure of distributing kernels into bins, the number of bins at every client is $m = \left \lceil \frac{N}{M} \right \rceil$. Furthermore, it can be concluded that $\gB_1$ includes $M$ kernels with the largest weights while the bin $\gB_m$ includes $N-(m-1)M$ kernels with lowest weights. The $k$-th client assigns the weight $u_{jk,t}$ at time $t$ to $\gB_j$ defined as $u_{jk,t} = \sum_{\kappa_i \in \gB_j}{w_{ik,t}}$. The $k$-th client draws one of the bins according to the probability mass function (PMF) $\vq_{k,t}$ defined as
\begin{align}
    q_{jk,t} = (1-\xi_k)\frac{u_{jk,t}}{U_{k,t}} + \frac{\xi_k}{m}, \forall j \in [m] \label{eq:10}
\end{align}
where $U_{k,t} = \sum_{j=1}^{m}{u_{jk,t}}$ and $0 < \xi_k \le 1$ is an exploration rate determined by the $k$-th client. Let $I_{k,t}$ be the index of the chosen bin by the $k$-th client at time $t$. The PMF in \eqref{eq:10} constitutes trade-off between exploitation and exploration. According to the first term in the right hand side of \eqref{eq:10}, it is more probable that the $k$-th client draws a bin which includes kernels with larger weights $w_{ik,t}$. Hence, it is more probable that the $k$-th client collaborates in updating the global parameters of a kernel with larger weight $w_{ik,t}$. Let $\gS_{k,t}$ denotes the set which includes the indices of kernels in the chosen bin at time $t$. 
\begin{algorithm}[tb]
	\caption{The $k$-th client kernel subset selection at time $t$.}
	\label{alg:1}
	\begin{algorithmic}
		\STATE {\bfseries Input:}{Weights $w_{ik,t}$, $\forall i \in [N]$, parameter $M$ and exploration rate $0<\xi_k\le 1$. }
		\STATE Sort the kernels in descending order with respect to weights $\{w_{ik,t}\}_{i=1}^{N}$.
		\STATE Obtain the index sequence $s_1, \ldots, s_N$ such that $w_{s_bk,t} \le w_{s_ak,t}$ if $b>a$, $\forall a,b \in [N]$.
		\STATE Open bin $\gB_1$ and initialize $j=1$.
		\FORALL{$i \in [N]$, the $k$-th client}
		\IF{the bin $\gB_j$ includes less than $M$ kernels}
		\STATE Adds the $s_i$-th kernel to $\gB_j$.
		\ELSE
		\STATE Opens new bin $\gB_{j+1}$, adds the $s_i$-th kernel to $\gB_{j+1}$ and updates $j \leftarrow j+1$.
		\ENDIF
		\ENDFOR
		\STATE Draw an index $I_{k,t}$ via PMF $\vq_{k,t}$ in \eqref{eq:10}.
		\STATE {\bfseries Output:} $\gS_{k,t}$: indices set of kernels in the selected bin $\gB_{I_{k,t}}$
	\end{algorithmic}
\end{algorithm}
The \Algref{alg:1} summarizes the procedure that the $k$-th client determines the set $\gS_{k,t}$. According to \Algref{alg:1}, kernel subset selection is personalized since each client chooses its own subset of kernels to update their parameters.

Let $p_{ik,t}$ denotes the probability that $i \in \gS_{k,t}$. Then $p_{ik,t} = q_{b_i k,t}$ where $b_i$ is the index of the bin which includes the $i$-th kernel. The $k$-th client updates global parameters locally as follows
\begin{align}
    \boldsymbol{\theta}_{ik,t+1} = \boldsymbol{\theta}_{i,t} - \eta  \frac{\nabla \gL(\boldsymbol{\theta}_{i,t}^\top \vz_i(\vx_{k,t}),y_{k,t})}{p_{ik,t}} \1_{i \in \gS_{k,t}} \label{eq:11}
\end{align}
where $\1_{i \in \gS_{k,t}}$ denotes an indicator function and it is $1$ when $i \in \gS_{k,t}$. The update rule in \eqref{eq:11} implies that when $i \notin \gS_{k,t}$, the $k$-th client does not update $\boldsymbol{\theta}_{i,t}$ (i.e. $\boldsymbol{\theta}_{ik,t+1} = \boldsymbol{\theta}_{i,t}$). Therefore, the $k$-th client sends $\boldsymbol{\theta}_{ik,t+1}$ to the server only if $i \in \gS_{k,t}$. Therefore, at each time instant, each client needs to send at most $2MD$ parameters to the server. Let $\gC_{i,t}$ be a set of client indices such that $k \in \gC_{i,t}$ if the $k$-th client sends $\boldsymbol{\theta}_{ik,t+1}$ to the server. Upon aggregating updates from clients, the server updates $\boldsymbol{\theta}_{i,t}$ as
\begin{align}
    \boldsymbol{\theta}_{i,t+1} = \boldsymbol{\theta}_{i,t} - \frac{1}{K} \sum_{k \in \gC_{i,t}}{(\boldsymbol{\theta}_{i,t}-\boldsymbol{\theta}_{ik,t+1})} = \boldsymbol{\theta}_{i,t} - \frac{\eta}{K} \sum_{k=1}^{K}{\frac{\nabla \gL(\boldsymbol{\theta}_{i,t}^\top \vz_i(\vx_{k,t}),y_{k,t})}{p_{ik,t}} \1_{i \in \gS_{k,t}}}. \label{eq:13}
\end{align}

\begin{algorithm}[tb]
	\caption{Personalized Online Federated Multi-Kernel Learning (POF-MKL)}
	\label{alg:2}
	\begin{algorithmic}
		\STATE {\bfseries Input:}{Kernels $\kappa_{i}$, $i=1,...,N$, learning rate $\eta>0$ and the number of random features $D$. }
		\STATE \textbf{Initialize:} $\vtheta_{i,1}=\mathbf{0}$, ${w}_{ik,1}=1$, $\forall i \in [N]$, $\forall k \in [K]$.
		\FOR{$t=1,\ldots,T$}
		\STATE The server transmits the global parameters $\hat{\boldsymbol{\Theta}}_t=[\vtheta_{1,t},\ldots,\vtheta_{N,t}]$ to all clients.
		\FORALL{$k \in [K]$, the $k$th client}
		\STATE Receive one datum $\vx_{k,t}$.
		\STATE Predicts $\hat{f}(\vx_{k,t};\hat{\boldsymbol{\Theta}}_t,\vw_{k,t})$ via \eqref{eq:9}.
		\STATE Calculates losses $\gL(\hat{f}_{\text{RF},it}(\vx_{k,t};\vtheta_{i,t}),y_{k,t})$, $\forall i \in [N]$.
		\STATE Updates ${w}_{ik,t+1}$, $\forall i \in [N]$ via \eqref{eq:12}.
		\STATE Selects a subset of kernel indices $\gS_{k,t}$ using \Algref{alg:1}.
		\STATE Updates ${\vtheta}_{ik,t+1}$, $\forall i \in \gS_{k,t}$ via \eqref{eq:11} and sends ${\vtheta}_{ik,t+1}$, $\forall i \in \gS_{k,t}$ to the server.
		\ENDFOR
		\STATE The server updates $\vtheta_{i,t+1}$, $\forall i \in [N]$ via \eqref{eq:13}. 
		\ENDFOR
	\end{algorithmic}
\end{algorithm}
\Algref{alg:2} summarizes the proposed personalized online federated multi-kernel learning algorithm called POF-MKL. It is useful to note that using our proposed POF-MKL, the server cannot find the gradients $\nabla \gL(\boldsymbol{\theta}_{i,t}^\top \vz_i(\vx_{k,t}),y_{k,t})$ from updates received from clients. Instead, the server can find $\nabla \gL(\boldsymbol{\theta}_{i,t}^\top \vz_i(\vx_{k,t}),y_{k,t})/p_{ik,t}$ where $p_{ik,t}$ is a time-varying value determined locally by the $k$-th client. This can promote the privacy of the proposed POF-MKL since exchanging the gradients can be hazardous to the privacy of federated learning (see e.g. \cite{Zhu2019,Geiping2020}).

\textbf{Complexity.} Each client needs to store $d$-dimensional $D$ random feature vectors for each kernel. Therefore, the memory requirement of each client to implement function approximation using POF-MKL is $\gO(dND)$. Using POF-MKL, at each time instant, each client needs to perform $\gO(dND)$ operations including inner products and summations. Furthermore, when $\xi_k < 1$, in order to choose a subset of kernels, the $k$-th client needs to sort kernels which imposes worst case computational complexity of $\gO(N\log N)$. However, when $\xi_k=1$, according to PMF in \eqref{eq:10}, the $k$-th client chooses one bin uniformly at random and as a result in this case the $k$-th client does not need to sort kernels. Therefore, setting $\xi_k < 1$, the computational complexity for the $k$-th client is $\gO(dND + N\log N)$ while setting $\xi_k=1$, the computational complexity of the $k$-th client at each time instant is $\gO(dND)$.

\subsection{Regret Analysis} \label{sec:analysis}
The present section analyzes the regret of the proposed POF-MKL. Specifically, two types of regret $\gR_{k,T}$ and $\gR_{s,T}$ are considered for the $k$-th client and the server, respectively. 
The performance of the $k$-th client utilizing POF-MKL is analyzed in terms of regret defined as
\begin{align}
    \gR_{k,T} = \sum_{t=1}^{T}{\gL(\hat{f}(\vx_{k,t};\hat{\boldsymbol{\Theta}}_t,\vw_{k,t}),y_{k,t})} - \min_{i \in [N]}{\sum_{t=1}^{T}{\gL(\hat{f}_{\text{RF},it}(\vx_{k,t};\vtheta_{i,t}),y_{k,t})}} \label{eq:14}
\end{align}
where $\gR_{k,T}$ measures the cumulative difference between the loss of the $k$-th client and the loss of the RF approximation of the kernel with minimum loss among all kernels' RF approximations. Let $\alpha_{ik,t}^*$, $\forall t \in [T]$, $\forall k \in [K]$ represents the optimal coefficients associated with the $i$-th kernel such that $f_i^*(\vx) = \sum_{t=1}^{T}{\sum_{k=1}^{K}{\alpha_{ik,t}^* \kappa_i(\vx,\vx_{k,t})}}$. Then the best function approximator is defined as $f^*(\cdot) \in \argmin_{f_i^*, i\in [N]} \sum_{t=1}^{T}\sum_{k=1}^{K}{{\gL(f_i^*(\vx_{k,t}),y_{k,t})}}$. Furthermore, the regret of the server is defined as the cumulative difference between the loss of POF-MKL and that of the best function approximator over all data samples distributed among clients which can be expressed as
\begin{align}
    \gR_{s,T} = \sum_{t=1}^{T}{\sum_{k=1}^{K}{\gL(\hat{f}(\vx_{k,t};\hat{\boldsymbol{\Theta}}_t,\vw_{k,t}),y_{k,t})}} - \sum_{t=1}^{T}{\sum_{k=1}^{K}{\gL(f^*(\vx_{k,t}),y_{k,t})}}. \label{eq:18}
\end{align}
In order to analyze the regret of POF-MKL, suppose that the following assumptions hold true:

\textbf{(as1)} $\gL(\vtheta_{i,t}^\top \vz_i(\vx_{k,t}),y_{k,t})$, $\forall k \in [K]$ is convex with respect to $\vtheta_{i,t}$ at each time instant $t$.\\
\textbf{(as2)} For $\vtheta$ in a bounded set satisfying $\|\vtheta\| \le C$, the loss function and its gradient are bounded as $0 \le \gL(\vtheta^\top \vz_i(\vx_{k,t}),y_{k,t}) \le 1$ and $\|\nabla \gL(\vtheta^\top \vz_i(\vx_{k,t}),y_{k,t})\| \le L$. Moreover, each data sample is bounded as $\|\vx_{k,t}\| \le 1$, $\forall k \in [K]$, $\forall t \in [T]$.\\
\textbf{(as3)} Kernels $\kappa_i(\cdot)$, $\forall i \in [N]$ are shift-invariant with $\kappa_i(\mathbf{0})=1$, $\forall i \in [N]$.

The following theorem investigates the regret of the $k$-th client according to the $k$-th client data. The proof of the following Theorem can be found in Appendix \ref{B}.
\begin{theorem} \label{th:2}
Under (as1)--(as3), the regret of the $k$-th client with respect to the best kernel satisfies
\begin{align}
    \sum_{t=1}^{T}{\gL(\hat{f}(\vx_{k,t};\hat{\boldsymbol{\Theta}}_t,\vw_{k,t}),y_{k,t})} - \min_{i \in [N]}{\sum_{t=1}^{T}{\gL(\hat{f}_{\text{RF},it}(\vx_{k,t};\vtheta_{i,t}),y_{k,t})}} \le \frac{\ln N}{\eta_k} + \frac{\eta_k}{2}T. \label{eq:16}
\end{align}
\end{theorem}
Theorem \ref{th:2} shows that by setting $\eta_k = \gO\left(\frac{1}{\sqrt{T}}\right)$, the $k$-th client achieves sub-linear regret of $\gO(\sqrt{T})$. Furthermore, Theorem \ref{th:2} shows that POF-MKL can deal with heterogeneous data among clients since the regret of each client defined in \eqref{eq:14} is calculated with respect to the corresponding client data. 
The following theorem studies the regret of the server with respect to the best function approximator. The proof can be found in Appendix \ref{C}.
\begin{theorem} \label{th:3}
Let $i^* := \argmin_{i \in [N]}{\sum_{t=1}^{T}{\sum_{k=1}^{K}{\gL(f_i^*(\vx_{k,t}),y_{k,t})}}}$ and $\sigma_i$ be the second Fourier moment of the $i$-th kernel. Under (as1)--(as3), the regret of the server with respect to the best function approximator satisfies
\begin{align}
    & \sum_{t=1}^{T}{\sum_{k=1}^{K}{\gL(\hat{f}(\vx_{k,t};\hat{\boldsymbol{\Theta}}_t,\vw_{k,t}),y_{k,t})}} - \sum_{t=1}^{T}{\sum_{k=1}^{K}{\gL(f^*(\vx_{k,t}),y_{k,t})}} \nonumber \\ \le & \frac{KC^2}{2\eta} + \frac{\eta}{2} \sum_{t=1}^{T}{\sum_{k=1}^{K}{\frac{L^2}{p_{i^*k,t}}}} + \sum_{k=1}^{K}{\left(\frac{\ln N}{\eta_k} + \frac{\eta_k}{2}T\right)} + \epsilon LKTC \label{eq:19}
\end{align}
with probability at least $1- 2^8 \left(\frac{\sigma_{i^*}}{\epsilon}\right)^2\exp\left(-\frac{D\epsilon^2}{4(d+2)}\right)$ where $C:=\max_{i \in [N]} \sum_{t=1}^{T}{\sum_{k=1}^{K}{\alpha_{ik,t}^*}}$.
\end{theorem}
As it can be inferred from \eqref{eq:19}, the regret of the server with respect to the best function approximator depends on $\frac{1}{p_{i^*k,t}}$. From \eqref{eq:10} and the fact that $p_{ik,t} = q_{b_i k,t}$, it can be concluded that $p_{i^*k,t} > \frac{\xi_k}{m}$. Thus, setting $\xi_k = \gO(1)$, then $p_{ik,t}>\gO(\frac{M}{N})$. The regret bound in \eqref{eq:19} shows that setting $\eta = \gO\left(\sqrt{\frac{M}{NT}}\right)$ and $\epsilon = \eta_k = \frac{1}{\sqrt{T}}$, $\forall k \in [K]$, the server obtains regret of $\gO\left(\sqrt{\frac{N}{M}T}\right)$ with probability at least $1- 2^8 \left(\frac{\sigma_{i^*}}{\epsilon}\right)^2\exp\left(-\frac{D\epsilon^2}{4(d+2)}\right)$. This shows that increasing $M$ tighten the regret bound and increasing $D$ increases the probability that the regret bound in \eqref{eq:19} holds true. However, using POF-MKL, each client needs to transmit $MD$ parameters at each time instant. Since both $M$ and $D$ are determined by the algorithm POF-MKL, this shows that POF-MKL can provide flexibility to tighten regret bound while the available clients-to-server communication bandwidth can afford transmission of clients' updates to the server. It is useful to mention that choosing larger value for $\xi_k$ increases the lower bound of $p_{i^*k,t}$ and as a result the optimal choice for $\xi_k$ in terms of regret is $\xi_k=1$. However, choosing smaller values for $\xi_k$ makes the value of $p_{ik,t}$ more dependent on weights $\{w_{ik,t}\}_{k=1}^{K}$ (c.f. \eqref{eq:10}). Therefore, choosing smaller values for $\xi_k$ makes $p_{ik,t}$ less predictable. This makes estimating $\nabla \mathcal L(\boldsymbol \theta_{i,t}^\top z_i(x_{k,t}),y_{k,t})$ given $\nabla \mathcal L(\boldsymbol \theta_{i,t}^\top z_i(x_{k,t}),y_{k,t})/ p_{ik,t}$ more difficult which leads to better protection of privacy.

\textbf{Comparison with personalized federated learning.} In order to deal with heterogeneous data among clients, personalized federated learning has been studied extensively in the literature (see \cite{Smith2017,Corinzia2019,Deng2020,Fallah2020,Hanzely2020,Dinh2020,Li2021,Li2021c,Shamsian2021,Collins2021,Acar2021,Marfoq2021,Zhang2021,Achituve2021,Sun2021,Chen2022}). Utilizing model-agnostic meta-learning \cite{Finn2017}, personalized federated learning algorithms have been proposed in \cite{Fallah2020,Acar2021}. In \cite{Shamsian2021,Chen2022}, personalized federated learning algorithms have been designed by learning hyper-networks \cite{Ha2017}. In \cite{Marfoq2021}, a personalized model is a linear combination of a set of shared component models such that each client constructs its personalized mixture of models. However, in aforementioned personalized federated learning works, clients are assumed to store a dataset to perform local updates with. Therefore, when clients are not able to store data in batch and they have to make a decision upon receiving a new data sample, aforementioned works in personalized federated learning cannot guarantee sub-linear regret for clients. However, according to Theorems \ref{th:2} and \ref{th:3}, POF-MKL provides sub-linear regret for clients when clients cannot store data in batch and make decision in an online fashion.

\textbf{Comparison with online federated learning \cite{Mitra2021}.} Fed-OMD algorithm has been proposed in \cite{Mitra2021} which enables clients to perform their learning task in an online and federated fashion while it is proved that Fed-OMD enjoys sub-linear regret when the loss function is convex with respect to parameters required to be learnt at each time instant. The proposed POF-MKL differs from Fed-OMD in the sense that Fed-OMD cannot guarantee sub-linear regret when it comes to performing the online learning task with RF approximations of multiple kernels since the loss function $\gL(\sum_{i=1}^{N}{w_{i,t}\vtheta_{i,t}^\top \vz_i(\vx)},y)$ is not convex with respect to both $\vtheta_{i,t}$ and $w_{i,t}$. However, according to Theorems \ref{th:2} and \ref{th:3}, the proposed POF-MKL guarantees sub-linear regret. 
    
\textbf{Comparison with \cite{Hong2021}.} Online federated MKL algorithms called vM-KOFL and eM-KOFL have been presented in \cite{Hong2021}. Both POF-MKL and algorithms in \cite{Hong2021} exploit random feature approximation to alleviate computational complexity of online kernel learning. Furthermore, both POF-MKL and algorithms in \cite{Hong2021} learn a linear combination of kernels. The proposed POF-MKL has the following advantages and innovations compared to vM-KOFL and eM-KOFL: i) The proposed POF-MKL allows clients to learn their own personalized combination of kernels (c.f. \eqref{eq:9}). As it is proved in Theorem \ref{th:2}, the proposed POF-MKL can deal with heterogeneous data among clients in the sense that using POF-MKL each client guarantees sub-linear regret with respect to the best kernel RF approximation according to the corresponding client data. However, both vM-KOFL and eM-KOFL are not able to provide such guarantee. ii) Using vM-KOFL, each client needs to send $(N+1)D$ parameters to the server. However, using the proposed POF-MKL, each client needs to send $MD$ parameters to the server such that $M \le N$ is determined by POF-MKL and can be chosen to be much smaller than $N$. iii) In eM-KOFL, the server chooses a kernel at each time instant and clients send their local updates associated with the chosen kernel by the server. The proposed POF-MKL provides more flexibility compared to eM-KOFL in the sense that using POF-MKL each client can send local updates of $M \ge 1$ kernels to the server. And each client chooses its own subset of kernels to send their updates to the server. Therefore, even though POF-MKL sets $M$ to $1$, it is possible that at a time instant the server receives updates associated with all kernels in the dictionary. It is useful to mention that using eM-KOFL the cumulative regret of all clients is sub-linear with respect to the best kernel RF approximation with probability $1-\delta$ where $0 < \delta \le 1$. However, utilizing the update rule in \eqref{eq:11}, using the proposed POF-MKL, each client obtains sub-linear regret with respect to RF approximation of its best kernel with probability 1.  

\section{Experiments} \label{sec:exp}
We tested the performance of the proposed POF-MKL for online regression task through a set of experiments. The performance of POF-MKL is compared with the baselines PerFedAvg \cite{Fallah2020}, OFSKL \cite{Mitra2021}, OFMKL-Avg \cite{Mitra2021}, vM-KOFL \cite{Hong2021} and eM-KOFL \cite{Hong2021}. PerFedAvg refers to the personalized federated averaging algorithm in \cite{Fallah2020}. In the experiments, PerFedAvg employs a fully connected feedforward neural network model. More information about the implementation of PerFedAvg can be found in Appendix \ref{sup:exp}. OFSKL and OFMKL-Avg are two variations of Fed-OMD \cite{Mitra2021}. OFSKL leverages Fed-OMD \cite{Mitra2021} when a single radial basis function (RBF) with bandwidth of $10$ is employed to perform the learning task. In OFMKL-Avg, kernels are learned independently from each other using Fed-OMD \cite{Mitra2021} and the prediction is the average of approximations given by kernels. Moreover, vM-KOFL and eM-KOFL are online federated MKL algorithms of \cite{Hong2021} such that vM-KOFL requires transmission of all kernel updates at every time instant while eM-KOFL requires transmission of a kernel update at each time instant.
In the experiments, each client observes $500$ samples until the end of the learning task meaning that $T=500$. The performance of the proposed POF-MKL and other baselines are tested on the following real datasets downloaded from UCI machine learning repository \cite{Dua2017}: Naval \cite{Coraddu2016}, UJI \cite{Torres-Sospedra2014}, Air \cite{Zhang2017} and WEC \cite{Neshat2018}. More detailed information about datasets can be found in Appendix \ref{sup:exp}.
Data samples of Naval and UJI datasets are distributed i.i.d among clients. Data samples in Air and WEC datasets are distributed non-i.i.d among clients. More inforamtion about distributing data samples among clients can be found in Appendix \ref{sup:exp}. The number of clients for Naval, UJI, Air and WEC datasets are $23$, $42$, $240$ and $560$, respectively. The dictionary of kernels consists of $51$ RBFs with different bandwidth such that the bandwidth of the $i$-th kernel is $\sigma_i = 10^{\frac{2i-52}{25}}$. We consider the case where the clients-to-server communication bandwidth is limited such that at each time instant, the maximum number of parameters that a client is allowed to transmit to the server is $1000$. Furthermore, the memory and computational capability of clients are limited such that the maximum value can be picked for the number of random features $D$ is $100$. The experiments were carried for $20$ different sets of random feature vectors. The performance of algorithms is measured using average of mean squared error (MSE) defined as
\begin{align}
\text{MSE} = \frac{1}{20} \sum_{j=1}^{20}{ \frac{1}{KT}\sum_{t=1}^{T}{\sum_{k=1}^{K}{(\hat{y}_{jk,t}-y_{k,t})^2}}} \nonumber
\end{align}
where $\hat{y}_{jk,t}$ denotes the prediction of the $k$-th client at time instant $t$ corresponding to the $j$-th set of random feature vectors. Learning rates are set to $\eta = \eta_k = \frac{1}{\sqrt{T}}$, $\forall k$. Also, exploration rates are set to $\xi_k=1$, $\forall k$. The performance of POF-MKL with different $\xi_k$ is studied in Appendix \ref{sup:exp}. Codes are available at \url{https://github.com/pouyamghari/POF-MKL}.

Table \ref{table:1} presents the MSE and run time performance of online federated kernel learning algorithms on real datasets. Run time refers to average total run time of clients to perform online learning task on the entire data samples that they observe. In Table \ref{table:1}, $M$ refers to the number of kernels whose updates are sent to the server after prediction at each time instant. And, $D$ is the number of random features. Comparing MSE of POF-MKL with that of OFMKL-Avg, it can be concluded that learning the weights to combine kernels provides higher accuracy than averaging kernels' predictions. Table \ref{table:1} shows that POF-MKL with $M=1$ provides lower MSE than eM-KOFL. Using eM-KOFL, at each time instant, the server receives updates belong to only one kernel. However, using POF-MKL with $M=1$, each client sends an update belongs to a kernel which is selected by the client. Therefore, the server receives updates associated with different kernels even though $M=1$. Therefore, experimental results show the effectiveness of the personalized kernel selection provided by POF-MKL. It can be observed that POF-MKL with $M=25$ obtains lower MSE than those of POF-MKL with $M=51$ and vM-KOFL. Since each client is allowed to send at most $1000$ parameters per time instant, if clients send updates of all kernels at every time instant as this is the case in vM-KOFL, $D$ cannot be chosen to be greater than $9$. However, setting $M=25$, POF-MKL can set $D=20$ which can improve the accuracy of online regression task compared to the case where $D=9$. Note that according to Theorem \ref{th:3}, increase in $D$ increases the probability that the server achieves sub-linear regret with respect to the best function approximator. Furthermore, POF-MKL with $M=51$ achieves lower MSE than vM-KOFL even if data samples are distributed i.i.d among clients. This shows that the proposed POF-MKL can better cope with heterogeneous data among clients which is in agreement with theoretical results in Theorem \ref{th:2}. In fact, the optimal combination of kernels can be different across clients. Using POF-MKL, each client constructs its own personalized combination of kernels which results in lower MSE compared to vM-KOFL. The proposed POF-MKL with $M=1$ and $M=25$ runs faster than eM-KOFL. In fact, using POF-MKL, clients only need to update parameters associated with $M$ kernels while employing vM-KOFL and eM-KOFL, clients have to update parameters of all kernels. Moreover, POF-MKL obtains lower MSE than PerFedAvg. Note that since clients are not able to store data in batch, at each time instant clients update PerFedAvg's model using only the newly observed data sample. Therefore, convergence of PerFedAvg is not guaranteed. Experimental results show that POF-MKL achieves higher accuracy than PerFedAvg in online regression task when it is not possible for clients to store data in batch. 
Since OFSKL employs only a pre-selected single kernel, OFSKL runs faster than POF-MKL. However, utilizing multiple kernels enables POF-MKL to obtain lower MSE than that of OFSKL. In fact, using POF-MKL clients learn a linear combination of kernels which is proved to enjoy sub-linear regret with respect to the best kernel in hindsight while employing OFSKL clients have to make predictions using a pre-selected kernel. 
\begin{table}[t]
\caption{MSE($\times 10^{-3}$) and run time of online federated learning algorithms on real datasets.}
\label{table:1}
\begin{center}
\begin{tabular}{l|c|c||c|c|c|c||c|c|c|c}
\toprule
    &   &   &\multicolumn{4}{c}{MSE($\times 10^{-3}$)}  &\multicolumn{4}{c}{Run time(s)} \\
{\bf Algorithms}    &$M$    &$D$    &{Naval}    &UJI   &{Air}  &{WEC} &{Naval}    &UJI   &{Air}  &{WEC}
\\ \hline
PerFedAvg  &-    &-   &$118.60$    &$63.03$  &$13.68$    &$77.33$  &$44.59$ &$41.67$   &$37.40$   &$33.56$ \\
OFSKL  &$1$    &$100$   &$77.77$    &$61.82$  &$13.65$    &$87.87$  &$0.07$ &$0.06$   &$0.08$   &$0.06$ \\
OFMKL-Avg  &$51$    &$9$    &$33.25$    &$55.44$  &$10.63$    &$34.01$   &$1.51$    &$1.73$   &$0.91$   &$0.47$ \\
vM-KOFL &$51$   &$9$   &$26.42$ &$51.50$  &$10.58$    &$25.17$  &$2.01$ &$2.22$   &$1.37$   &$0.67$ \\
eM-KOFL &$1$    &$100$    &$28.64$  &$61.08$  &$21.94$    &$20.14$  &$2.27$ &$10.13$   &$1.45$   &$1.70$ \\ \hline
POF-MKL &$1$    &$100$   &$\mathbf{16.16}$  &$\mathbf{33.02}$  &$\mathbf{9.27}$   &$\mathbf{11.44}$ &$1.22$ &$9.02$   &$1.25$   &$1.10$\\
POF-MKL &$25$   &$20$   &$16.82$    &$37.34$   &$9.34$   &$11.58$   &$0.69$ &$2.29$   &$0.63$   &$0.52$\\
POF-MKL &$51$   &$9$    &$16.65$    &$41.00$   &$9.38$   &$11.97$   &$0.82$ &$1.07$   &$0.81$   &$0.65$\\
\bottomrule
\end{tabular}
\end{center}
\end{table}
Furthermore, Figure \ref{fig:regret} illustrates the average regret of clients when clients employ vM-KOFL and the proposed POF-MKL with different $M$ parameters. From Figure \ref{fig:regret}, it can be observed that the proposed POF-MKL achieves sub-linear regret.
\begin{figure}
\centering
\subfigure[Naval dataset.]{%
  \centering
  \includegraphics[scale=.22]{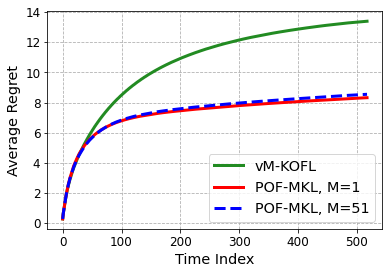}
    }
\quad
\subfigure[UJI dataset.]{%
  \centering
  \includegraphics[scale=.22]{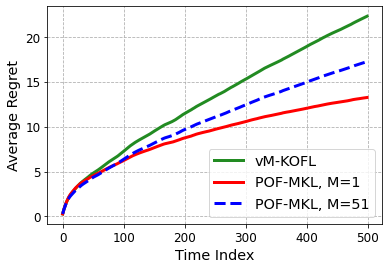}
    }
\quad
\subfigure[Air dataset.]{%
  \centering
  \includegraphics[scale=.22]{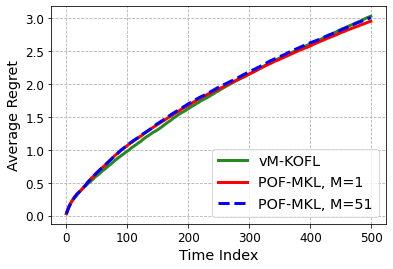}
    }
\quad
\subfigure[WEC dataset.]{%
  \centering
  \includegraphics[scale=.22]{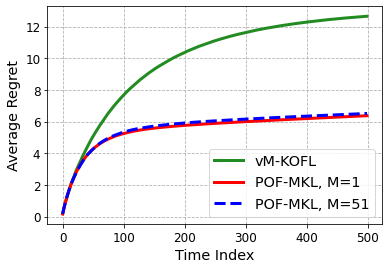}
	}
\caption{Average regret of clients.}
\label{fig:regret}
\end{figure}

\section{Conclusions}
The present paper proposed a personalized online federated MKL algorithm called POF-MKL based on RF approximation. Employing the proposed POF-MKL, each client updates the parameters of a subset of kernels which alleviates the computational complexity of the client as well as communication cost of sending updated parameters of kernels. Theoretical analysis proved that using POF-MKL, each client achieves sub-linear regret with respect to the RF approximation of its best kernel in hindsight which indicates that POF-MKL can deal heterogeneous data among clients. While each client updates a subset of kernels, it was proved that the server achieves sub-linear regret with respect to the best function approximator. Experiments on real datasets showcased the advantages of POF-MKL compared with other online federated kernel learning algorithms.

\section*{Acknowledgement}
This work is supported by NSF ECCS 2207457. PI Yanning Shen also receives support from Microsoft Academic Grant Award for AI Research. Contact: Yanning Shen (yannings@uci.edu).

\bibliographystyle{plainnat}
\bibliography{References}

\begin{thebibliography}{53}
\providecommand{\natexlab}[1]{#1}
\providecommand{\url}[1]{\texttt{#1}}
\expandafter\ifx\csname urlstyle\endcsname\relax
  \providecommand{\doi}[1]{doi: #1}\else
  \providecommand{\doi}{doi: \begingroup \urlstyle{rm}\Url}\fi

\bibitem[Acar et~al.(2021)Acar, Zhao, Zhu, Matas, Mattina, Whatmough, and
  Saligrama]{Acar2021}
Durmus Alp~Emre Acar, Yue Zhao, Ruizhao Zhu, Ramon Matas, Matthew Mattina, Paul
  Whatmough, and Venkatesh Saligrama.
\newblock Debiasing model updates for improving personalized federated
  training.
\newblock In \emph{Proceedings of International Conference on Machine
  Learning}, volume 139, pages 21--31, Jul 2021.

\bibitem[Achituve et~al.(2021)Achituve, Shamsian, Navon, Chechik, and
  Fetaya]{Achituve2021}
Idan Achituve, Aviv Shamsian, Aviv Navon, Gal Chechik, and Ethan Fetaya.
\newblock Personalized federated learning with gaussian processes.
\newblock In \emph{Proceedings of International Conference on Neural
  Information Processing Systems}, volume~34, pages 8392--8406, Dec 2021.

\bibitem[Bengio et~al.(2006)Bengio, Delalleau, and Roux]{Bengio2006}
Yoshua Bengio, Olivier Delalleau, and Nicolas~L. Roux.
\newblock The curse of highly variable functions for local kernel machines.
\newblock In \emph{Advances in Neural Information Processing Systems}, pages
  107--114, May 2006.

\bibitem[Bouboulis et~al.(2018)Bouboulis, Chouvardas, and
  Theodoridis]{Bouboulis2018}
Pantelis Bouboulis, Symeon Chouvardas, and Sergios Theodoridis.
\newblock Online distributed learning over networks in rkh spaces using random
  fourier features.
\newblock \emph{IEEE Transactions on Signal Processing}, 66\penalty0
  (7):\penalty0 1920--1932, Apr 2018.

\bibitem[Chen and Chao(2022)]{Chen2022}
Hong-You Chen and Wei-Lun Chao.
\newblock On bridging generic and personalized federated learning for image
  classification.
\newblock In \emph{International Conference on Learning Representations}, 2022.

\bibitem[Collins et~al.(2021)Collins, Hassani, Mokhtari, and
  Shakkottai]{Collins2021}
Liam Collins, Hamed Hassani, Aryan Mokhtari, and Sanjay Shakkottai.
\newblock Exploiting shared representations for personalized federated
  learning.
\newblock In \emph{Proceedings of the International Conference on Machine
  Learning}, volume 139, pages 2089--2099, Jul 2021.

\bibitem[Coraddu et~al.(2016)Coraddu, Oneto, Ghio, Savio, Anguita, and
  Figari]{Coraddu2016}
Andrea Coraddu, Luca Oneto, Aessandro Ghio, Stefano Savio, Davide Anguita, and
  Massimo Figari.
\newblock Machine learning approaches for improving condition-based maintenance
  of naval propulsion plants.
\newblock \emph{Proceedings of the Institution of Mechanical Engineers, Part M:
  Journal of Engineering for the Maritime Environment}, 230\penalty0
  (1):\penalty0 136--153, 2016.

\bibitem[Corinzia et~al.(2019)Corinzia, Beuret, and Buhmann]{Corinzia2019}
Luca Corinzia, Ami Beuret, and Joachim~M. Buhmann.
\newblock Variational federated multi-task learning, 2019.

\bibitem[Deng et~al.(2020)Deng, Kamani, and Mahdavi]{Deng2020}
Yuyang Deng, Mohammad~Mahdi Kamani, and Mehrdad Mahdavi.
\newblock Adaptive personalized federated learning, 2020.

\bibitem[Dinh et~al.(2020)Dinh, Tran, and Nguyen]{Dinh2020}
Canh~T. Dinh, Nguyen~H. Tran, and Tuan~Dung Nguyen.
\newblock Personalized federated learning with moreau envelopes.
\newblock In \emph{Proceedings of International Conference on Neural
  Information Processing Systems}, page 21394–21405, Dec 2020.

\bibitem[Dua and Graff(2017)]{Dua2017}
Dheeru Dua and Casey Graff.
\newblock {UCI} machine learning repository, 2017.

\bibitem[Fallah et~al.(2020)Fallah, Mokhtari, and Ozdaglar]{Fallah2020}
Alireza Fallah, Aryan Mokhtari, and Asuman Ozdaglar.
\newblock Personalized federated learning with theoretical guarantees: A
  model-agnostic meta-learning approach.
\newblock In \emph{Advances in Neural Information Processing Systems},
  volume~33, pages 3557--3568, Dec 2020.

\bibitem[Finn et~al.(2017)Finn, Abbeel, and Levine]{Finn2017}
Chelsea Finn, Pieter Abbeel, and Sergey Levine.
\newblock Model-agnostic meta-learning for fast adaptation of deep networks.
\newblock In \emph{Proceedings of International Conference on Machine
  Learning}, volume~70, pages 1126--1135, Aug 2017.

\bibitem[Geiping et~al.(2020)Geiping, Bauermeister, Dr\"{o}ge, and
  Moeller]{Geiping2020}
Jonas Geiping, Hartmut Bauermeister, Hannah Dr\"{o}ge, and Michael Moeller.
\newblock Inverting gradients - how easy is it to break privacy in federated
  learning?
\newblock In \emph{Advances in Neural Information Processing Systems},
  volume~33, pages 16937--16947, 2020.

\bibitem[Ghari and Shen(2020)]{Ghari2020}
Pouya~M Ghari and Yanning Shen.
\newblock Online multi-kernel learning with graph-structured feedback.
\newblock In \emph{Proceedings of the International Conference on Machine
  Learning}, volume 119, pages 3474--3483, Jul 2020.

\bibitem[Ghari and Shen(2022)]{Ghari2022}
Pouya~M Ghari and Yanning Shen.
\newblock Graph-assisted communication-efficient ensemble federated learning.
\newblock \emph{arXiv preprint arXiv:2202.13447}, 2022.

\bibitem[Gogineni et~al.(2022)Gogineni, Werner, Huang, and Kuh]{Gogineni2022}
Vinay~Chakravarthi Gogineni, Stefan Werner, Yih-Fang Huang, and Anthony Kuh.
\newblock Communication-efficient online federated learning framework for
  nonlinear regression.
\newblock In \emph{IEEE International Conference on Acoustics, Speech and
  Signal Processing (ICASSP)}, pages 5228--5232, May 2022.

\bibitem[Ha et~al.(2017)Ha, Dai, and Le]{Ha2017}
David Ha, Andrew~M. Dai, and Quoc~V. Le.
\newblock Hypernetworks.
\newblock In \emph{International Conference on Learning Representations}, 2017.

\bibitem[Hamer et~al.(2020)Hamer, Mohri, and Suresh]{Hamer2020}
Jenny Hamer, Mehryar Mohri, and Ananda~Theertha Suresh.
\newblock {F}ed{B}oost: A communication-efficient algorithm for federated
  learning.
\newblock In \emph{Proceedings of International Conference on Machine
  Learning}, volume 119, pages 3973--3983, Jul 2020.

\bibitem[Hanzely et~al.(2020)Hanzely, Hanzely, Horv\'{a}th, and
  Richt\'{a}rik]{Hanzely2020}
Filip Hanzely, Slavom\'{\i}r Hanzely, Samuel Horv\'{a}th, and Peter
  Richt\'{a}rik.
\newblock Lower bounds and optimal algorithms for personalized federated
  learning.
\newblock In \emph{Proceedings of International Conference on Neural
  Information Processing Systems}, page 2304–2315, Dec 2020.

\bibitem[Hoi et~al.(2013)Hoi, Jin, Zhao, and Yang]{Hoi2013}
Steven C.~H. Hoi, Rong Jin, Peilin Zhao, and Tianbao Yang.
\newblock Online multiple kernel classification.
\newblock \emph{Machine Learning}, 90:\penalty0 289–316, Feb 2013.

\bibitem[Hong and Chae(2021)]{Hong2021}
Songnam Hong and Jeongmin Chae.
\newblock Communication-efficient randomized algorithm for multi-kernel online
  federated learning.
\newblock \emph{IEEE Transactions on Pattern Analysis and Machine
  Intelligence}, Nov 2021.

\bibitem[Kairouz et~al.(2021)Kairouz, McMahan, Avent, Bellet, Bennis, Bhagoji,
  Bonawitz, Charles, Cormode, Cummings, D’Oliveira, Eichner, Rouayheb, Evans,
  Gardner, Garrett, Gascón, Ghazi, Gibbons, Gruteser, Harchaoui, He, He, Huo,
  Hutchinson, Hsu, Jaggi, Javidi, Joshi, Khodak, Konecný, Korolova,
  Koushanfar, Koyejo, Lepoint, Liu, Mittal, Mohri, Nock, Özgür, Pagh, Qi,
  Ramage, Raskar, Raykova, Song, Song, Stich, Sun, Suresh, Tramèr, Vepakomma,
  Wang, Xiong, Xu, Yang, Yu, Yu, and Zhao]{Kairouz2021}
Peter Kairouz, H.~Brendan McMahan, Brendan Avent, Aurélien Bellet, Mehdi
  Bennis, Arjun~Nitin Bhagoji, Kallista Bonawitz, Zachary Charles, Graham
  Cormode, Rachel Cummings, Rafael G.~L. D’Oliveira, Hubert Eichner, Salim~El
  Rouayheb, David Evans, Josh Gardner, Zachary Garrett, Adrià Gascón, Badih
  Ghazi, Phillip~B. Gibbons, Marco Gruteser, Zaid Harchaoui, Chaoyang He, Lie
  He, Zhouyuan Huo, Ben Hutchinson, Justin Hsu, Martin Jaggi, Tara Javidi,
  Gauri Joshi, Mikhail Khodak, Jakub Konecný, Aleksandra Korolova, Farinaz
  Koushanfar, Sanmi Koyejo, Tancrède Lepoint, Yang Liu, Prateek Mittal,
  Mehryar Mohri, Richard Nock, Ayfer Özgür, Rasmus Pagh, Hang Qi, Daniel
  Ramage, Ramesh Raskar, Mariana Raykova, Dawn Song, Weikang Song, Sebastian~U.
  Stich, Ziteng Sun, Ananda~Theertha Suresh, Florian Tramèr, Praneeth
  Vepakomma, Jianyu Wang, Li~Xiong, Zheng Xu, Qiang Yang, Felix~X. Yu, Han Yu,
  and Sen Zhao.
\newblock Advances and open problems in federated learning.
\newblock \emph{Foundations and Trends® in Machine Learning}, 14\penalty0
  (1–2):\penalty0 1--210, 2021.

\bibitem[Kloft et~al.(2011)Kloft, Brefeld, Sonnenburg, and Zien]{Kloft2011}
Marius Kloft, Ulf Brefeld, S\"{o}ren Sonnenburg, and Alexander Zien.
\newblock Lp-norm multiple kernel learning.
\newblock \emph{Journal of Machine Learning Research}, 12:\penalty0 953–997,
  Jul 2011.

\bibitem[Konečný et~al.(2017)Konečný, McMahan, Yu, Richtárik, Suresh, and
  Bacon]{konency2017}
Jakub Konečný, H.~Brendan McMahan, Felix~X. Yu, Peter Richtárik,
  Ananda~Theertha Suresh, and Dave Bacon.
\newblock Federated learning: Strategies for improving communication
  efficiency, 2017.

\bibitem[Kuh(2021)]{Kuh2021}
Anthony Kuh.
\newblock Real time kernel learning for sensor networks using principles of
  federated learning.
\newblock In \emph{Asia-Pacific Signal and Information Processing Association
  Annual Summit and Conference (APSIPA ASC)}, pages 2089--2093, Dec 2021.

\bibitem[Li et~al.(2020)Li, Sahu, Talwalkar, and Smith]{Li2020}
Tian Li, Anit~Kumar Sahu, Ameet Talwalkar, and Virginia Smith.
\newblock Federated learning: Challenges, methods, and future directions.
\newblock \emph{IEEE Signal Processing Magazine}, 37\penalty0 (3):\penalty0
  50--60, 2020.

\bibitem[Li et~al.(2021{\natexlab{a}})Li, Hu, Beirami, and Smith]{Li2021c}
Tian Li, Shengyuan Hu, Ahmad Beirami, and Virginia Smith.
\newblock Ditto: Fair and robust federated learning through personalization.
\newblock In \emph{Proceedings of International Conference on Machine
  Learning}, volume 139, pages 6357--6368, Jul 2021{\natexlab{a}}.

\bibitem[Li et~al.(2021{\natexlab{b}})Li, Jiang, Zhang, Kamp, and Dou]{Li2021}
Xiaoxiao Li, Meirui Jiang, Xiaofei Zhang, Michael Kamp, and Qi~Dou.
\newblock Fed{BN}: Federated learning on non-{IID} features via local batch
  normalization.
\newblock In \emph{International Conference on Learning Representations},
  2021{\natexlab{b}}.

\bibitem[Lu et~al.(2016)Lu, Hoi, Wang, Zhao, and Liu]{Lu2016}
Jing Lu, Steven~C.H. Hoi, Jialei Wang, Peilin Zhao, and Zhi-Yong Liu.
\newblock Large scale online kernel learning.
\newblock \emph{Journal of Machine Learning Research}, 17\penalty0
  (47):\penalty0 1--43, 2016.

\bibitem[Lv et~al.(2021)Lv, Wang, Liu, and Liu]{Lv2021}
Shaogao Lv, Junhui Wang, Jiankun Liu, and Yong Liu.
\newblock Improved learning rates of a functional lasso-type svm with sparse
  multi-kernel representation.
\newblock In \emph{Advances in Neural Information Processing Systems},
  volume~34, pages 21467--21479, Dec 2021.

\bibitem[Marfoq et~al.(2021)Marfoq, Neglia, Bellet, Kameni, and
  Vidal]{Marfoq2021}
Othmane Marfoq, Giovanni Neglia, Aur{\'e}lien Bellet, Laetitia Kameni, and
  Richard Vidal.
\newblock Federated multi-task learning under a mixture of distributions.
\newblock In \emph{Proceedings of International Conference on Neural
  Information Processing Systems}, volume~34, pages 15434--15447, Dec 2021.

\bibitem[McMahan et~al.(2017)McMahan, Moore, Ramage, Hampson, and
  y~Arcas]{McMahan2017}
Brendan McMahan, Eider Moore, Daniel Ramage, Seth Hampson, and Blaise~Aguera
  y~Arcas.
\newblock {Communication-Efficient Learning of Deep Networks from Decentralized
  Data}.
\newblock In \emph{Proceedings of International Conference on Artificial
  Intelligence and Statistics}, volume~54, pages 1273--1282, Apr 2017.

\bibitem[Mitra et~al.(2021)Mitra, Hassani, and Pappas]{Mitra2021}
Aritra Mitra, Hamed Hassani, and George~J. Pappas.
\newblock Online federated learning.
\newblock In \emph{IEEE Conference on Decision and Control (CDC)}, pages
  4083--4090, Dec 2021.

\bibitem[Neshat et~al.(2018)Neshat, Alexander, Wagner, and Xia]{Neshat2018}
Mehdi Neshat, Bradley Alexander, Markus Wagner, and Yuanzhong Xia.
\newblock A detailed comparison of meta-heuristic methods for optimising wave
  energy converter placements.
\newblock In \emph{Proceedings of the Genetic and Evolutionary Computation
  Conference}, page 1318–1325, Jul 2018.

\bibitem[Rahimi and Recht(2007)]{Rahimi2007}
Ali Rahimi and Benjamin Recht.
\newblock Random features for large-scale kernel machines.
\newblock In \emph{Proceedings of International Conference on Neural
  Information Processing Systems}, pages 1177--1184, Dec 2007.

\bibitem[Rothchild et~al.(2020)Rothchild, Panda, Ullah, Ivkin, Stoica,
  Braverman, Gonzalez, and Arora]{Rothchild2020}
Daniel Rothchild, Ashwinee Panda, Enayat Ullah, Nikita Ivkin, Ion Stoica,
  Vladimir Braverman, Joseph Gonzalez, and Raman Arora.
\newblock {F}etch{SGD}: Communication-efficient federated learning with
  sketching.
\newblock In \emph{Proceedings of International Conference on Machine
  Learning}, volume 119, pages 8253--8265, Jul 2020.

\bibitem[Rudi and Rosasco(2017)]{Rudi2017}
Alessandro Rudi and Lorenzo Rosasco.
\newblock Generalization properties of learning with random features.
\newblock In \emph{Proceedings of International Conference on Neural
  Information Processing Systems}, page 3218–3228, 2017.

\bibitem[Sahoo et~al.(2014)Sahoo, Hoi, and Li]{Sahoo2014}
Doyen Sahoo, Steven~C.H. Hoi, and Bin Li.
\newblock Online multiple kernel regression.
\newblock In \emph{Proceedings of ACM SIGKDD International Conference on
  Knowledge Discovery and Data Mining}, page 293–302, 2014.

\bibitem[Sahoo et~al.(2019)Sahoo, Hoi, and Li]{Sahoo2019}
Doyen Sahoo, Steven C.~H. Hoi, and Bin Li.
\newblock Large scale online multiple kernel regression with application to
  time-series prediction.
\newblock \emph{ACM Transactions on Knowledge Discovery from Data}, 13\penalty0
  (1), Jan 2019.

\bibitem[Scholkopf and Smola(2001)]{Scholkopf2001}
Bernhard Scholkopf and Alexander~J. Smola.
\newblock \emph{Learning with Kernels: Support Vector Machines, Regularization,
  Optimization, and Beyond}.
\newblock MIT Press, Cambridge, MA, USA, 2001.

\bibitem[Shamsian et~al.(2021)Shamsian, Navon, Fetaya, and
  Chechik]{Shamsian2021}
Aviv Shamsian, Aviv Navon, Ethan Fetaya, and Gal Chechik.
\newblock Personalized federated learning using hypernetworks.
\newblock In \emph{Proceedings of International Conference on Machine
  Learning}, volume 139, pages 9489--9502, Jul 2021.

\bibitem[Shen et~al.(2019)Shen, Chen, and Giannakis]{Shen2019}
Yanning Shen, Tianyi Chen, and Georgios~B. Giannakis.
\newblock Random feature-based online multi-kernel learning in environments
  with unknown dynamics.
\newblock \emph{Journal of Machine Learning Research}, 20\penalty0
  (1):\penalty0 773--808, Jan 2019.

\bibitem[Smith et~al.(2017)Smith, Chiang, Sanjabi, and Talwalkar]{Smith2017}
Virginia Smith, Chao-Kai Chiang, Maziar Sanjabi, and Ameet~S Talwalkar.
\newblock Federated multi-task learning.
\newblock In \emph{Advances in Neural Information Processing Systems},
  volume~30, pages 4424--4434, Dec 2017.

\bibitem[Sonnenburg et~al.(2006)Sonnenburg, R\"{a}tsch, Sch\"{a}fer, and
  Sch\"{o}lkopf]{Sonnenburg2006}
S\"{o}ren Sonnenburg, Gunnar R\"{a}tsch, Christin Sch\"{a}fer, and Bernhard
  Sch\"{o}lkopf.
\newblock Large scale multiple kernel learning.
\newblock \emph{Journal of Machine Learning Research}, 7:\penalty0 1531–1565,
  Dec 2006.

\bibitem[Sun et~al.(2021)Sun, Huo, YANG, and Bai]{Sun2021}
Benyuan Sun, Hongxing Huo, YI~YANG, and Bo~Bai.
\newblock Partialfed: Cross-domain personalized federated learning via partial
  initialization.
\newblock In \emph{Proceedings of International Conference on Neural
  Information Processing Systems}, volume~34, pages 23309--23320, Dec 2021.

\bibitem[Torres-Sospedra et~al.(2014)Torres-Sospedra, Montoliu, Martínez-Usó,
  Avariento, Arnau, Benedito-Bordonau, and Huerta]{Torres-Sospedra2014}
Joaquín Torres-Sospedra, Raúl Montoliu, Adolfo Martínez-Usó, Joan~P.
  Avariento, Tomás~J. Arnau, Mauri Benedito-Bordonau, and Joaquín Huerta.
\newblock Ujiindoorloc: A new multi-building and multi-floor database for wlan
  fingerprint-based indoor localization problems.
\newblock In \emph{International Conference on Indoor Positioning and Indoor
  Navigation (IPIN)}, pages 261--270, Oct 2014.

\bibitem[Wahba(1990)]{Wahba1990}
Grace Wahba.
\newblock \emph{Spline Models for Observational Data}.
\newblock Society for Industrial and Applied Mathematics, 1990.

\bibitem[Williams and Seeger(2000)]{Williams2000}
Christopher K.~I. Williams and Matthias Seeger.
\newblock Using the nystr\"{o}m method to speed up kernel machines.
\newblock In \emph{Proceedings of the International Conference on Neural
  Information Processing Systems}, page 661–667, Jan 2000.

\bibitem[Zhang et~al.(2021)Zhang, Guo, Ma, Wang, Xu, and Wu]{Zhang2021}
Jie Zhang, Song Guo, Xiaosong Ma, Haozhao Wang, Wenchao Xu, and Feijie Wu.
\newblock Parameterized knowledge transfer for personalized federated learning.
\newblock In \emph{Proceedings of International Conference on Neural
  Information Processing Systems}, volume~34, pages 10092--10104, Dec 2021.

\bibitem[Zhang et~al.(2017)Zhang, Guo, Dong, He, Xu, and Chen]{Zhang2017}
Shuyi Zhang, Bin Guo, Anlan Dong, Jing He, Ziping Xu, and Song~Xi Chen.
\newblock Cautionary tales on air-quality improvement in beijing.
\newblock \emph{Proceedings of the Royal Society A: Mathematical, Physical and
  Engineering Sciences}, 473\penalty0 (2205):\penalty0 20170457, 2017.

\bibitem[Zhang and Liao(2019)]{Zhang2019}
Xiao Zhang and Shizhong Liao.
\newblock Incremental randomized sketching for online kernel learning.
\newblock In \emph{Proceedings of International Conference on Machine
  Learning}, pages 7394--7403, Jun 2019.

\bibitem[Zhu et~al.(2019)Zhu, Liu, and Han]{Zhu2019}
Ligeng Zhu, Zhijian Liu, and Song Han.
\newblock Deep leakage from gradients.
\newblock In \emph{Advances in Neural Information Processing Systems},
  volume~32, 2019.

\end{thebibliography}

\section*{Checklist}


\begin{enumerate}

\item For all authors...
\begin{enumerate}
  \item Do the main claims made in the abstract and introduction accurately reflect the paper's contributions and scope?
    \answerYes{}
  \item Did you describe the limitations of your work?
    \answerYes{In section \ref{sec:exp}, we mention that the proposed multi-kernl leanring algorithm POF-MKL runs slower than single kernel learning algorithm OFSKL.}
  \item Did you discuss any potential negative societal impacts of your work?
    \answerNA{We do not foresee any direct societal impact.}
  \item Have you read the ethics review guidelines and ensured that your paper conforms to them?
    \answerYes{}
\end{enumerate}

\item If you are including theoretical results...
\begin{enumerate}
  \item Did you state the full set of assumptions of all theoretical results?
    \answerYes{Please see section \ref{sec:analysis} where we explicitly explain the assumptions.}
        \item Did you include complete proofs of all theoretical results?
    \answerYes{Complete proofs of theoretical results are presented in Appendices \ref{B} and \ref{C}.}
\end{enumerate}

\item If you ran experiments...
\begin{enumerate}
  \item Did you include the code, data, and instructions needed to reproduce the main experimental results (either in the supplemental material or as a URL)?
    \answerYes{We provide code, data, and instructions to reproduce the main experimental results in the supplementary material.}
  \item Did you specify all the training details (e.g., data splits, hyperparameters, how they were chosen)?
    \answerYes{In section \ref{sec:exp} and Appendix \ref{sup:exp}, detailed information about experimental setup is provided.}
        \item Did you report error bars (e.g., with respect to the random seed after running experiments multiple times)?
    \answerYes{We did the experiments for $20$ sets of random features. In Appendix \ref{sup:exp}, the standard deviation of MSE is reported.}
        \item Did you include the total amount of compute and the type of resources used (e.g., type of GPUs, internal cluster, or cloud provider)?
    \answerYes{In Table \ref{table:1}, we report run time of algorithms. Also, in Appendix \ref{sup:exp}, we explain that all experiments were carried out using Intel(R) Core(TM) i7-10510U CPU @ 1.80 GHz 2.30 GHz processor with a 64-bit {Windows} operating system.}
\end{enumerate}

\item If you are using existing assets (e.g., code, data, models) or curating/releasing new assets...
\begin{enumerate}
  \item If your work uses existing assets, did you cite the creators?
    \answerYes{In section \ref{sec:exp}, we cited all datasets used for experiments.}
  \item Did you mention the license of the assets?
    \answerNA{}
  \item Did you include any new assets either in the supplemental material or as a URL?
    \answerNo{}
  \item Did you discuss whether and how consent was obtained from people whose data you're using/curating?
    \answerYes{In section \ref{sec:exp}, we point out that datasets are downloaded from UCI Machine Learning Repository.}
  \item Did you discuss whether the data you are using/curating contains personally identifiable information or offensive content?
    \answerNA{}
\end{enumerate}

\item If you used crowdsourcing or conducted research with human subjects...
\begin{enumerate}
  \item Did you include the full text of instructions given to participants and screenshots, if applicable?
    \answerNA{}
  \item Did you describe any potential participant risks, with links to Institutional Review Board (IRB) approvals, if applicable?
    \answerNA{}
  \item Did you include the estimated hourly wage paid to participants and the total amount spent on participant compensation?
    \answerNA{}
\end{enumerate}

\end{enumerate}


\newpage

\appendix

\section{Proof of Theorem \ref{th:2}} \label{B}
Since $W_{k,t} = \sum_{i=1}^{N}{w_{ik,t}}$ and according to \eqref{eq:12}, we can write
\begin{align}
    \frac{W_{k,t+1}}{W_{k,t}} = \sum_{i=1}^{N}{\frac{w_{k,t+1}}{W_{k,t}}} = \sum_{i=1}^{N}{\frac{w_{ik,t}}{W_{k,t}}\exp\left(-\eta_k \gL(\hat{f}_{\text{RF},it}(\vx_{k,t};\vtheta_{i,t}),y_{k,t})\right)}. \label{eq:12ap}
\end{align}
Using the inequality $e^{-x} \le 1-x+\frac{1}{2}x^2, \forall x \ge 0$, the upper bound of \eqref{eq:12ap} can be obtained as
\begin{align}
    \frac{W_{k,t+1}}{W_{k,t}} \le \sum_{i=1}^{N}{\frac{w_{ik,t}}{W_{k,t}}\left(1- \eta_k \gL(\hat{f}_{\text{RF},it}(\vx_{k,t};\vtheta_{i,t}),y_{k,t}) + \frac{\eta_k^2}{2} \gL^2(\hat{f}_{\text{RF},it}(\vx_{k,t};\vtheta_{i,t}),y_{k,t})\right)}. \label{eq:13ap}
\end{align}
Using the inequality $1+x \le e^x$ and taking the logarithm from both sides of \eqref{eq:13ap}, we get
\begin{align}
    \ln \frac{W_{k,t+1}}{W_{k,t}} \le \sum_{i=1}^{N}{\frac{w_{ik,t}}{W_{k,t}}\left( - \eta_k \gL(\hat{f}_{\text{RF},it}(\vx_{k,t};\vtheta_{i,t}),y_{k,t}) + \frac{\eta_k^2}{2} \gL^2(\hat{f}_{\text{RF},it}(\vx_{k,t};\vtheta_{i,t}),y_{k,t})\right)}. \label{eq:14ap}
\end{align}
According to (as2), $\gL^2(\hat{f}_{\text{RF},it}(\vx_{k,t};\vtheta_{i,t}),y_{k,t}) \le 1$. Therefore, from \eqref{eq:14ap}, we can conclude that
\begin{align}
    \ln \frac{W_{k,t+1}}{W_{k,t}} \le \sum_{i=1}^{N}{\frac{w_{ik,t}}{W_{k,t}}\left( \frac{\eta_k^2}{2} - \eta_k \gL(\hat{f}_{\text{RF},it}(\vx_{k,t};\vtheta_{i,t}),y_{k,t}) \right)}. \label{eq:15ap}
\end{align}
Summing \eqref{eq:15ap} over time, we arrive at
\begin{align}
    \ln \frac{W_{k,T+1}}{W_{k,1}} \le \sum_{t=1}^{T}{\sum_{i=1}^{N}{\frac{w_{ik,t}}{W_{k,t}}\left( \frac{\eta_k^2}{2} - \eta_k \gL(\hat{f}_{\text{RF},it}(\vx_{k,t};\vtheta_{i,t}),y_{k,t}) \right)}}. \label{eq:16ap}
\end{align}
Moreover, for any $i\in [N]$, $\ln \frac{W_{k,T+1}}{W_{k,1}}$ can be lower bounded as
\begin{align}
    \ln \frac{W_{k,T+1}}{W_{k,1}} \ge \ln \frac{w_{ik,T+1}}{W_{k,1}} = - \eta_k \sum_{t=1}^{T}{\gL(\hat{f}_{\text{RF},it}(\vx_{k,t};\vtheta_{i,t}),y_{k,t})} - \ln N. \label{eq:17ap}
\end{align}
Combining \eqref{eq:16ap} with \eqref{eq:17ap}, we obtain
\begin{align}
    \sum_{t=1}^{T}{\sum_{i=1}^{N}{\frac{w_{ik,t}}{W_{k,t}}\gL(\hat{f}_{\text{RF},it}(\vx_{k,t};\vtheta_{i,t}),y_{k,t})}} - \sum_{t=1}^{T}{\gL(\hat{f}_{\text{RF},it}(\vx_{k,t};\vtheta_{i,t}),y_{k,t})} \le \frac{\ln N}{\eta_k} + \frac{\eta_k}{2}T. \label{eq:18ap}
\end{align}
Since the loss function $\gL(\cdot,\cdot)$ is convex, using \eqref{eq:18ap} and Jensen inequality we can write
\begin{align}
    \sum_{t=1}^{T}{\gL\left(\sum_{i=1}^{N}{\frac{w_{ik,t}}{W_{k,t}}\hat{f}_{\text{RF},it}(\vx_{k,t};\vtheta_{i,t})},y_{k,t}\right)} - \sum_{t=1}^{T}{\gL(\hat{f}_{\text{RF},it}(\vx_{k,t};\vtheta_{i,t}),y_{k,t})} \le \frac{\ln N}{\eta_k} + \frac{\eta_k}{2}T \label{eq:19ap}
\end{align}
which proves the Theorem \ref{th:2}.

\section{Proof of Theorem \ref{th:3}}\label{C}
In order to prove Theorem \ref{th:3}, the following Lemma is used as a stepstone.
\begin{lemma} \label{th:1}
Let $\alpha_{ik,t}^*$, $\forall t \in [T]$, $\forall k \in [K]$ represents the optimal coefficients associated with the $i$-th kernel such that $f_i^*(\vx) = \sum_{t=1}^{T}{\sum_{k=1}^{K}{\alpha_{ik,t}^* \kappa_i(\vx,\vx_{k,t})}}$. And $\hat{f}_i^*(\vx) = (\vtheta_i^*)^\top \vz_i(\vx)$ denotes the best RF-based estimator associated with the $i$-th kernel such that $\vtheta_i^* = \sum_{t=1}^{T}{\sum_{k=1}^{K}{\alpha_{ik,t}^* \vz_i(\vx_{k,t})}}$. Under assumptions (as1)--(as3), using POF-MKL, the RF approximation of the $i$-th kernel satisfies
\begin{align}
    & \sum_{t=1}^{T}{\sum_{k=1}^{K}{\gL(\hat{f}_{\text{RF},it}(\vx_{k,t};\vtheta_{i,t}),y_{k,t})}} - \sum_{t=1}^{T}{\sum_{k=1}^{K}{\gL(\hat{f}_i^*(\vx_{k,t}),y_{k,t})}} \nonumber \\ \le& \frac{K \|\vtheta_i^*\|^2}{2\eta} + \frac{\eta}{2} \sum_{t=1}^{T}{\sum_{k=1}^{K}{\frac{L^2}{p_{ik,t}}}}. \label{eq:15}
\end{align}
\end{lemma}
\begin{proof}
Let $\ell_{ik,t}$ be the importance sampling loss estimate defined as
\begin{align}
    \ell_{ik,t} = \frac{\gL(\boldsymbol{\theta}_{i,t}^\top \vz_i(\vx_{k,t}),y_{k,t})}{p_{ik,t}} \1_{i \in \gS_{k,t}}. \label{eq:1ap}
\end{align}
Then according to \eqref{eq:11}, for any fixed $\vtheta$, it can be written that
\begin{align}
    \left \|\frac{1}{K}\sum_{k=1}^{K}{\vtheta_{ik,t+1}} - \vtheta \right \|^2 &= \left \|\frac{1}{K}\sum_{k=1}^{K}{(\vtheta_{i,t}-\eta \nabla\ell_{ik,t})} - \vtheta \right \|^2 = \left \|\vtheta_{i,t} - \vtheta - \frac{\eta}{K}\sum_{k=1}^{K}{\nabla\ell_{ik,t}} \right \|^2 \nonumber \\ &= \|\vtheta_{i,t} - \vtheta\|^2 - \frac{2\eta}{K}\left(\sum_{k=1}^{K}{\nabla^\top \ell_{ik,t}}\right)(\vtheta_{i,t} - \vtheta) + \left\| \frac{\eta}{K}\sum_{k=1}^{K}{\nabla \ell_{ik,t}} \right\|^2 \label{eq:2ap}
\end{align}
According to the convexity of the loss function $\gL(\vtheta^\top \vx, y)$ with respect to $\vtheta$ as stated in (as1), we find
\begin{align}
    \gL(\vtheta_{i,t}^\top \vz_i(\vx_{k,t}),y_{k,t}) - \gL(\vtheta^\top \vz_i(\vx_{k,t}),y_{k,t}) \le \nabla^\top \gL(\vtheta_{i,t}^\top \vz_i(\vx_{k,t}),y_{k,t})(\vtheta_{i,t}-\vtheta). \label{eq:3ap}
\end{align}
Multiplying both sides of \eqref{eq:3ap} by $\frac{\1_{i \in \gS_{k,t}}}{p_{ik,t}}$, we get
\begin{align}
    & \left(\frac{\gL(\vtheta_{i,t}^\top \vz_i(\vx_{k,t}),y_{k,t})}{p_{ik,t}} - \frac{\gL(\vtheta^\top \vz_i(\vx_{k,t}),y_{k,t})}{p_{ik,t}}\right)\1_{i \in \gS_{k,t}} \nonumber \\ \le & \frac{\nabla^\top \gL(\vtheta_{i,t}^\top \vz_i(\vx_{k,t}),y_{k,t})}{p_{ik,t}}(\vtheta_{i,t}-\vtheta)\1_{i \in \gS_{k,t}}. \label{eq:4ap}
\end{align}
Summing \eqref{eq:4ap} over $k$, $\forall k \in [K]$, we arrive at
\begin{align}
    & \sum_{k=1}^{K}{\left(\frac{\gL(\vtheta_{i,t}^\top \vz_i(\vx_{k,t}),y_{k,t})}{p_{ik,t}} - \frac{\gL(\vtheta^\top \vz_i(\vx_{k,t}),y_{k,t})}{p_{ik,t}}\right)\1_{i \in \gS_{k,t}}} \nonumber \\ \le & \left(\sum_{k=1}^{K}{\frac{\nabla^\top \gL(\vtheta_{i,t}^\top \vz_i(\vx_{k,t}),y_{k,t})}{p_{ik,t}}\1_{i \in \gS_{k,t}}}\right) (\vtheta_{i,t}-\vtheta). \label{eq:5ap}
\end{align}
Based on the definition of $\ell_{ik,t}$, \eqref{eq:5ap} can be rewritten as
\begin{align}
    \sum_{k=1}^{K}{\ell_{ik,t}} - \sum_{k=1}^{K}{ \frac{\gL(\vtheta^\top \vz_i(\vx_{k,t}),y_{k,t})}{p_{ik,t}}\1_{i \in \gS_{k,t}}} \le \left(\sum_{k=1}^{K}{\nabla^\top \ell_{ik,t}}\right)(\vtheta_{i,t} - \vtheta). \label{eq:6ap}
\end{align}
According to \eqref{eq:2ap}, \eqref{eq:6ap} is equivalent to
\begin{align}
    & \sum_{k=1}^{K}{\ell_{ik,t}} - \sum_{k=1}^{K}{ \frac{\gL(\vtheta^\top \vz_i(\vx_{k,t}),y_{k,t})}{p_{ik,t}}\1_{i \in \gS_{k,t}}} \nonumber \\ \le & \frac{K}{2\eta}\left(\|\vtheta_{i,t} - \vtheta\|^2 - \left \|\frac{1}{K}\sum_{k=1}^{K}{\vtheta_{ik,t+1}} - \vtheta \right \|^2 + \left\| \frac{\eta}{K}\sum_{k=1}^{K}{\nabla \ell_{ik,t}} \right\|^2 \right) \label{eq:7ap}
\end{align}
Expectations of $\ell_{ik,t}$ and $\|\nabla \ell_{ik,t}\|^2$ with respect to $\1_{i \in \gS_{k,t}}$ can be calculated as
\begin{subequations} \label{eq:8ap}
\begin{align}
    \E_t[\ell_{ik,t}] &= \frac{\gL(\vtheta_{i,t}^\top \vz_i(\vx_{k,t}),y_{k,t})}{p_{ik,t}} p_{ik,t} = \gL(\vtheta_{i,t}^\top \vz_i(\vx_{k,t}),y_{k,t}) \label{eq:8apa} \\
    \E_t[\|\nabla \ell_{ik,t}\|^2] &= \frac{\|\nabla \gL(\vtheta_{i,t}^\top \vz_i(\vx_{k,t}),y_{k,t})\|^2}{p_{ik,t}^2} p_{ik,t} = \frac{\|\nabla \gL(\vtheta_{i,t}^\top \vz_i(\vx_{k,t}),y_{k,t})\|^2}{p_{ik,t}}. \label{eq:8apb}
\end{align}
\end{subequations}
Furthermore, using AM-GM inequality and \eqref{eq:8apb}, it can be concluded that
\begin{align}
    \E_t\left[\|\sum_{k=1}^{K}{\nabla \ell_{ik,t}}\|^2\right] \le \E_t\left[K\sum_{k=1}^{K}{\|\nabla \ell_{ik,t}\|^2}\right] \le K \sum_{k=1}^{K}{\frac{\|\nabla \gL(\vtheta_{i,t}^\top \vz_i(\vx_{k,t}),y_{k,t})\|^2}{p_{ik,t}} }. \label{eq:9ap}
\end{align}
According to \eqref{eq:13}, considering the fact that
\begin{align}
    \|\vtheta_{i,t+1} - \vtheta\|^2 = \left \|\frac{1}{K}\sum_{k=1}^{K}{\vtheta_{ik,t+1}} - \vtheta \right \|^2 \nonumber
\end{align}
taking the expectation from \eqref{eq:7ap} with respect to $\1_{i\in \gS_{k,t}}$, $\forall k \in [K]$ leads to
\begin{align}
    & \sum_{k=1}^{K}{\gL(\vtheta_{i,t}^\top \vz_i(\vx_{k,t}),y_{k,t})} - \sum_{k=1}^{K}{\gL(\vtheta^\top \vz_i(\vx_{k,t}),y_{k,t})} \nonumber \\ \le & \frac{K}{2\eta}\left(\|\vtheta_{i,t} - \vtheta\|^2 - \|\vtheta_{i,t+1} - \vtheta\|^2\right) + \frac{\eta}{2} \sum_{k=1}^{K}{\frac{\|\nabla \gL(\vtheta_{i,t}^\top \vz_i(\vx_{k,t}),y_{k,t})\|^2}{p_{ik,t}} }. \label{eq:10ap}
\end{align}
According to (as2), we can conclude that $\|\nabla \gL(\vtheta_{i,t}^\top \vz_i(\vx_{k,t}),y_{k,t})\|^2 \le L^2$. Hence, summing \eqref{eq:10ap} over time, given the fact that $\vtheta_{i,1}=\mathbf{0}$, $\forall i \in [N]$, we get
\begin{align}
    & \sum_{t=1}^{T}{\sum_{k=1}^{K}{\gL(\vtheta_{i,t}^\top \vz_i(\vx_{k,t}),y_{k,t})}} - \sum_{t=1}^{T}{\sum_{k=1}^{K}{\gL(\vtheta^\top \vz_i(\vx_{k,t}),y_{k,t})}} \nonumber \\ \le & \frac{K}{2\eta}\left(\|\vtheta\|^2 - \|\vtheta_{i,T+1} - \vtheta\|^2\right) + \frac{\eta}{2} \sum_{t=1}^{T}{\sum_{k=1}^{K}{\frac{L^2}{p_{ik,t}}}}. \label{eq:11ap}
\end{align}
Replacing $\vtheta$ with $\vtheta_i^*$ and considering the fact that $\|\vtheta_{i,T+1} - \vtheta\|^2 \ge 0$, we obtain
\begin{align}
    \sum_{t=1}^{T}{\sum_{k=1}^{K}{\gL(\vtheta_{i,t}^\top \vz_i(\vx_{k,t}),y_{k,t})}} - \sum_{t=1}^{T}{\sum_{k=1}^{K}{\gL((\vtheta_i^*)^\top \vz_i(\vx_{k,t}),y_{k,t})}} \le \frac{K \|\vtheta_i^*\|^2}{2\eta} + \frac{\eta}{2} \sum_{t=1}^{T}{\sum_{k=1}^{K}{\frac{L^2}{p_{ik,t}}}} \nonumber
\end{align}
which proves the Lemma \ref{th:1}.
\end{proof}

In order to proof Theorem \ref{th:3}, we leverage the results obtained in the proofs of Lemma \ref{th:1} and Theorem \ref{th:2}. Since \eqref{eq:19ap} holds true for any $i$, summing \eqref{eq:19ap} over all $k\in [K]$, for any $i$ we can write
\begin{align}
    & \sum_{t=1}^{T}{\sum_{k=1}^{K}{\gL(\hat{f}(\vx_{k,t};\hat{\boldsymbol{\Theta}}_t,\vw_{k,t}),y_{k,t})}} - {\sum_{t=1}^{T}{\sum_{k=1}^{K}{\gL(\hat{f}_{\text{RF},it}(\vx_{k,t};\vtheta_{i,t}),y_{k,t})}}} \nonumber \\ \le & \sum_{k=1}^{K}{\left(\frac{\ln N}{\eta_k} + \frac{\eta_k}{2}T\right)}. \label{eq:20ap}
\end{align}
Combining \eqref{eq:20ap} with \eqref{eq:15}, we arrive at
\begin{align}
    & \sum_{t=1}^{T}{\sum_{k=1}^{K}{\gL(\hat{f}(\vx_{k,t};\hat{\boldsymbol{\Theta}}_t,\vw_{k,t}),y_{k,t})}} - \sum_{t=1}^{T}{\sum_{k=1}^{K}{\gL(\hat{f}_i^*(\vx_{k,t}),y_{k,t})}} \nonumber \\ \le & \frac{K \|\vtheta_i^*\|^2}{2\eta} + \frac{\eta}{2} \sum_{t=1}^{T}{\sum_{k=1}^{K}{\frac{L^2}{p_{ik,t}}}} + \sum_{k=1}^{K}{\left(\frac{\ln N}{\eta_k} + \frac{\eta_k}{2}T\right)}. \label{eq:21ap}
\end{align}
According to claim 1 in \cite{Rahimi2007}, it can be written that $\sup_{\vx,\vx^\prime} |\vz_i^\top(\vx)\vz_i(\vx^\prime) - \kappa_i(\vx,\vx^\prime)| \le \epsilon$ holds true with probability greater than
$
    1- 2^8 \left(\frac{\sigma_i}{\epsilon}\right)^2\exp\left(-\frac{D\epsilon^2}{4(d+2)}\right)
$
where $\sigma_i$ is the second Fourier moment of the $i$-th kernel $\kappa_i(\cdot)$. Furthermore, according to (as3), the loss function is $L$-Lipschitz continuous and as a result it can be written that
\begin{align}
    & \sum_{t=1}^{T}{\sum_{k=1}^{K}{|\gL(\hat{f}_i^*(\vx_{k,t}),y_{k,t}) - \gL(f_i^*(\vx_{k,t}),y_{k,t})|}} \nonumber \\ \le & \sum_{t=1}^{T}{\sum_{k=1}^{K}{L\left|\sum_{\tau=1}^{T}{\sum_{j=1}^{K}{\alpha_{ij,\tau}^* \vz_i^\top(\vx_{j,\tau})\vz_i(\vx_{k,t})}} - \sum_{\tau=1}^{T}{\sum_{j=1}^{K}{\alpha_{ij,\tau}^* \kappa_i(\vx_{j,\tau},\vx_{k,t})}}\right|}}. \label{eq:22ap}
\end{align}
Applying Cauchy-Schwartz inequality to the right hand side of \eqref{eq:22ap}, the left hand side of \eqref{eq:22ap} can be bounded from above as
\begin{align}
    & \sum_{t=1}^{T}{\sum_{k=1}^{K}{|\gL(\hat{f}_i^*(\vx_{k,t}),y_{k,t}) - \gL(f_i^*(\vx_{k,t}),y_{k,t})|}} \nonumber \\ \le & \sum_{t=1}^{T}{\sum_{k=1}^{K}{L\sum_{\tau=1}^{T}{\sum_{j=1}^{K}{|\alpha_{ij,\tau}^*||\vz_i^\top(\vx_{j,\tau})\vz_i(\vx_{k,t}) - \kappa_i(\vx_{j,\tau},\vx_{k,t})|}}}}. \label{eq:23ap}
\end{align}
Let $C:=\max_{i \in [N]} \sum_{t=1}^{T}{\sum_{k=1}^{K}{\alpha_{ik,t}^*}}$. Therefore, we can conclude that
\begin{align}
    \sum_{t=1}^{T}{\sum_{k=1}^{K}{|\gL(\hat{f}_i^*(\vx_{k,t}),y_{k,t}) - \gL(f_i^*(\vx_{k,t}),y_{k,t})|}} \le \epsilon LKTC \label{eq:24ap}
\end{align}
with probability at least $1- 2^8 \left(\frac{\sigma_i}{\epsilon}\right)^2\exp\left(-\frac{D\epsilon^2}{4(d+2)}\right)$. Moreover, using Triangle inequality, we can write
\begin{align}
    & \sum_{t=1}^{T}{\sum_{k=1}^{K}{\gL(\hat{f}_i^*(\vx_{k,t}),y_{k,t})}} - \sum_{t=1}^{T}{\sum_{k=1}^{K}{\gL(f_i^*(\vx_{k,t}),y_{k,t})}} \nonumber \\ \le & \left|\sum_{t=1}^{T}{\sum_{k=1}^{K}{\gL(\hat{f}_i^*(\vx_{k,t}),y_{k,t}) - \gL(f_i^*(\vx_{k,t}),y_{k,t})}}\right| \nonumber \\ \le & \sum_{t=1}^{T}{\sum_{k=1}^{K}{|\gL(\hat{f}_i^*(\vx_{k,t}),y_{k,t}) - \gL(f_i^*(\vx_{k,t}),y_{k,t})|}} \le \epsilon LKTC \label{eq:25ap}
\end{align}
which holds true with probability at least $1- 2^8 \left(\frac{\sigma_i}{\epsilon}\right)^2\exp\left(-\frac{D\epsilon^2}{4(d+2)}\right)$. Moreover, for $\vz_i^\top(\vx)\vz_i(\vx^\prime)$ we can write
\begin{align}
    \vz_i^\top(\vx)\vz_i(\vx^\prime) = \frac{1}{D}\sum_{j=1}^{D}{(\sin(\boldsymbol{\rho}_{i,j}^\top \vx)\sin(\boldsymbol{\rho}_{i,j}^\top \vx^\prime)+\cos(\boldsymbol{\rho}_{i,j}^\top \vx)\cos(\boldsymbol{\rho}_{i,j}^\top \vx^\prime))}. \label{eq:27ap}
\end{align}
Based on arithmetic-mean geometric-mean, \eqref{eq:27ap} can be relaxed to
\begin{align}
    \vz_i^\top(\vx)\vz_i(\vx^\prime) \le \frac{1}{D}\sum_{j=1}^{D}{\frac{1}{2}(\sin^2(\boldsymbol{\rho}_{i,j}^\top \vx)+\sin^2(\boldsymbol{\rho}_{i,j}^\top \vx^\prime)+\cos^2(\boldsymbol{\rho}_{i,j}^\top \vx)+\cos^2(\boldsymbol{\rho}_{i,j}^\top \vx^\prime))} = 1. \label{eq:28ap}
\end{align}
Thus, given the fact that $|\vz_i^\top(\vx)\vz_i(\vx^\prime)| \le 1$, $\|\vtheta_i^*\|^2$ can be bounded as
\begin{align}
    \|\vtheta_i^*\|^2 \le \sum_{t=1}^{T}{\sum_{k=1}^{K}{\sum_{\tau=1}^{T}{\sum_{j=1}^{K}{|\alpha_{ik,t}^*\alpha_{ij,\tau}^*\vz_i^\top(\vx_{j,\tau})\vz_i(\vx_{k,t})|}}}} \le C^2. \label{eq:26ap}
\end{align}
Combining \eqref{eq:25ap} and \eqref{eq:26ap} with \eqref{eq:21ap} yields
\begin{align}
    & \sum_{t=1}^{T}{\sum_{k=1}^{K}{\gL(\hat{f}(\vx_{k,t};\hat{\boldsymbol{\Theta}}_t,\vw_{k,t}),y_{k,t})}} - \sum_{t=1}^{T}{\sum_{k=1}^{K}{\gL(f_i^*(\vx_{k,t}),y_{k,t})}} \nonumber \\ \le & \frac{KC^2}{2\eta} + \frac{\eta}{2} \sum_{t=1}^{T}{\sum_{k=1}^{K}{\frac{L^2}{p_{ik,t}}}} + \sum_{k=1}^{K}{\left(\frac{\ln N}{\eta_k} + \frac{\eta_k}{2}T\right)} + \epsilon LKTC
\end{align}
which holds true for any $i \in [N]$ with probability at least $1- 2^8 \left(\frac{\sigma_i}{\epsilon}\right)^2\exp\left(-\frac{D\epsilon^2}{4(d+2)}\right)$. Therefore, this proves the Theorem \ref{th:3}.

\section{Supplementary Experimental Results and Details} \label{sup:exp}
This section presents further experimental results testing different aspects of the proposed algorithm POF-MKL. Moreover, this section provides more detailed information about experimental setup associated with results in section \ref{sec:exp}. The performance of federated kernel learning algorithms are tested on the following datasets:
\begin{itemize}
    \item \textbf{Naval:} The dataset consists of $11,500$ samples. Each sample has $15$ features of a a naval vessel. The goal is to predict lever position \cite{Coraddu2016}.
    \item \textbf{UJI:} The dataset consists of $21,000$ data samples. Each data sample has $520$ features which are WiFi fingerprints. The goal is to predict the geographical longitude associated with each data sample.
    \item \textbf{Air:} The dataset consists of $120,000$ samples with $14$ features including information related to air quality such as concentration of some chemicals in the air. Data samples are collected from $4$ different geographical sites. The goal is to predict the concentration of CO in the air \cite{Zhang2017}. For each site there are $30,000$ samples in the dataset.
    \item \textbf{WEC:} The dataset consists of $280,000$ samples with $48$ features of wave energy converters. Data samples are collected from $4$ different geographical sites. The goal is to predict total power output \cite{Neshat2018}. For each site, there are $70,000$ samples.
\end{itemize}
Data samples of Naval and UJI datasets are distributed i.i.d among clients. The number of clients for Naval and UJI datasets are $23$ and $42$, respectively. Data samples in Air and WEC datasets are distributed non-i.i.d among clients. The number of clients for Air and WEC datasets are $240$ and $560$, respectively. For both Air and WEC datasets, there are $4$ different geographical sites that each sample belongs to one of them. Each client observes $350$ samples from one site and $50$ samples from each of the rest of $3$ sites.
Moreover, PerFedAvg uses a feedforward neural network model. Each layer is a fully-connected dense layer with at most $20$ neurons. Neurons in hidden layers exploit ReLU activation functions. Since each client cannot transmit more than $1000$ parameters to the server, the number of hidden layers is determined in a way that the number of the neural network's parameters to be less than $1000$. The number of parameters depends on the number of features in data samples. Therefore, the number of hidden layers varies across different datasets. For each dataset, given the number of features, the maximum number of hidden layers with $20$ neurons is chosen.
All experiments were carried out using Intel(R) Core(TM) i7-10510U CPU @ 1.80 GHz 2.30 GHz processor with a 64-bit {Windows} operating system.

Table \ref{table:3} presents average MSE along with MSE standard deviation calculated over $20$ different sets of random feature vectors. As it can be seen from Table \ref{table:3}, the proposed POF-MKL provides lower standard deviation compared to all other baselines. This shows that the proposed POF-MKL is less sensitive to the choice of random features.
\begin{table}[t]
\caption{MSE($\times 10^{-3}$) and standard deviation($\times 10^{-3}$) of online federated learning algorithms on real datasets.}
\label{table:3}
\begin{center}
\begin{tabular}{l|c|c||c|c|c|c}
\toprule
{\bf Algorithms}    &$M$    &$D$  &{Naval}  &{UJI}   &{Air}  &{WEC}
\\ \hline
OFSKL  &$1$    &$100$    &$77.77\pm1.04$    &$61.82\pm2.76$  &$13.65\pm0.61$    &$87.87\pm3.93$ \\
OFMKL-Avg  &$51$    &$9$    &$33.25\pm1.46$ &$55.44\pm 2.48$  &$10.63\pm0.47$    &$34.01\pm1.52$ \\
vM-KOFL &$51$   &$9$    &$26.42\pm1.16$ &$51.50\pm 2.30$  &$10.58\pm0.47$    &$25.17\pm1.12$ \\
eM-KOFL &$1$    &$100$     &$28.64\pm1.32$  &$61.08\pm 2.73$  &$21.94\pm1.16$    &$20.14\pm0.93$ \\ \hline
POF-MKL &$1$    &$100$  &$\mathbf{16.16}\pm\mathbf{0.72}$   &$\mathbf{33.02}\pm \mathbf{1.48}$  &$\mathbf{9.27}\pm\mathbf{0.41}$   &$\mathbf{11.44}\pm\mathbf{0.52}$ \\
POF-MKL &$25$   &$20$     &$16.82\pm0.74$   &$37.34\pm 1.67$   &$9.34\pm 0.42$   &$11.58\pm0.53$\\
POF-MKL &$51$   &$9$   &$16.65\pm0.74$  &$41.00\pm 1.83$   &$9.38\pm 0.42$   &$11.97\pm0.55$\\
\bottomrule
\end{tabular}
\end{center}
\end{table}
Furthermore, Table  \ref{table:4} reports the average cumulative regret of clients along with the standard deviation of regret among clients. As it can be inferred from Table \ref{table:4}, the proposed POF-MKL obtains lower regret than other online federated MKL algorithms. Moreover, for Air and WEC datasets, the standard deviation of regret among clients associated with POF-MKL is considerably lower than those of other online federated MKL algorithms. Note that data samples in Air and WEC datasets are distributed non-i.i.d among clients. Therefore, the results in Table \ref{table:4} confirm that the proposed POF-MKL can better deal with heterogeneous data among clients.
\begin{table}[t]
\caption{Average regret and its standard deviation across clients for online federated MKL learning algorithms.}
\label{table:4}
\begin{center}
\begin{tabular}{l|c|c||c|c|c|c}
\toprule
{\bf Algorithms}    &$M$    &$D$  &{Naval}  &{UJI}   &{Air}  &{WEC}
\\ \hline
OFMKL-Avg  &$51$    &$9$    &$16.95\pm0.39$ &$24.23 \pm 20.33$  &$3.05\pm2.83$    &$17.07\pm14.52$ \\
vM-KOFL &$51$   &$9$    &$13.40\pm0.39$ &$22.28\pm15.56$  &$3.02\pm2.79$    &$12.65\pm11.53$ \\
eM-KOFL &$1$    &$100$     &$14.92\pm0.65$  &$27.26\pm19.79$  &$8.53\pm2.51$    &$10.43\pm8.28$ \\ \hline
POF-MKL &$1$    &$100$  &$\mathbf{8.33}\pm0.39$ &$\mathbf{13.23}\pm\mathbf{8.80}$  &$\mathbf{2.95}\pm\mathbf{2.08}$   &$\mathbf{6.37}\pm\mathbf{5.28}$ \\
POF-MKL &$25$   &$20$     &$8.67\pm0.39$    &$15.41\pm9.58$   &$2.98\pm 2.12$   &$6.41\pm5.39$\\
POF-MKL &$51$   &$9$   &$8.55\pm0.39$   &$17.23\pm10.07$   &$3.01\pm 2.15$   &$6.51\pm5.57$\\
\bottomrule
\end{tabular}
\end{center}
\end{table}

Figure \ref{fig:xi} illustrates the average regret of clients employing POF-MKL with the change in the value of exploration rate $\xi_k$ when the exploration rate of all clients are the same. In particular, Figure \ref{fig:xi} depicts the performance of POF-MKL for $M=1$ and $M=25$ with the change in $\xi_k$. According to the PMF $\vq_{k,t}$ defined in \eqref{eq:10}, the increase in $\xi_k$ leads to increase in exploration such that if $\xi_k=1$, the $k$-th client chooses a subset of kernels uniformly at random. Figure \ref{fig:xi} indicates that the optimal choice of $\xi_k$ in terms of regret depends on the dataset distributed among clients as well as the number of chosen kernels $M$. Moreover, the choice of $\xi_k$ is related to the computational complexity of executing POF-MKL by clients. Specifically, when $\xi_k < 1$, in order to choose a subset of kernels, the $k$-th client needs to sort kernels which imposes worst case computational complexity of $\gO(N\log N)$. However, when $\xi_k=1$, according to PMF in \eqref{eq:10}, the $k$-th client chooses one bin uniformly at random and as a result in this case the $k$-th client does not need to sort kernels. Also, it is useful to note that as it can be inferred from \eqref{eq:11}, clients can leverage the exploration rate $\xi_k$ to send their updates to the server without revealing both the gradient of loss and the loss of kernels which can promote the privacy of the proposed POF-MKL. 

\begin{figure}
\centering
\subfigure[Naval dataset.]{%
  \centering
  \includegraphics[scale=.3]{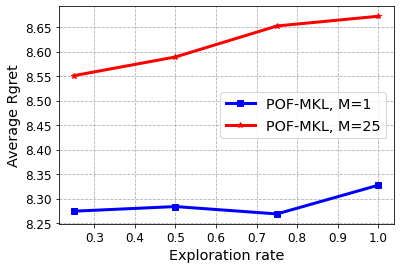}
    }
\quad
\subfigure[Air dataset.]{%
  \centering
  \includegraphics[scale=.3]{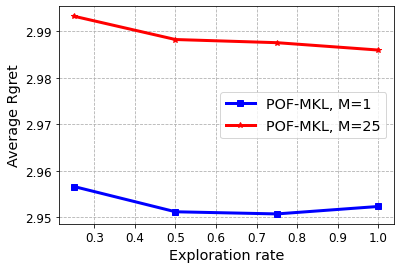}
    }
\quad
\subfigure[WEC dataset.]{%
  \centering
  \includegraphics[scale=.3]{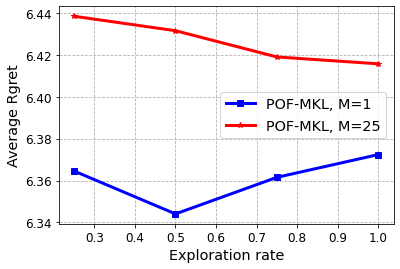}
	}
\caption{Average cumulative regret of clients with the change in the value of exploration rate ($\xi_k$).}
\label{fig:xi}
\end{figure}

\end{document}